\PassOptionsToPackage{table,xcdraw,dvipsnames}{xcolor}
\documentclass[sigconf]{acmart}
\AtBeginDocument{%
  }

\copyrightyear{2026}
\acmYear{2026}
\setcopyright{cc}
\setcctype{by}
\acmConference[KDD 2026] {Proceedings of the 32nd ACM SIGKDD Conference on Knowledge Discovery and Data Mining V.1}{August 9--13, 2025}{Jeju Island, Republic of Korea.}
\acmBooktitle{Proceedings of the 32nd ACM SIGKDD Conference on Knowledge Discovery and Data Mining V.1 (KDD 2026), August 9--13, 2025, Jeju Island, Republic of Korea}
\acmPrice{}
\acmDOI{10.1145/3770854.3780263}
\acmISBN{979-8-4007-2258-5/2026/08}
\usepackage{xcolor}
\usepackage{colortbl} 
\usepackage{url}

\usepackage{algorithmic}
\usepackage[ruled,vlined]{algorithm2e}
\usepackage[shortlabels]{enumitem}
\setlist[enumerate]{nosep}

\usepackage{amsmath}
\usepackage{mathtools}
\usepackage{amsthm}
\usepackage{bm}
\usepackage{mathrsfs}
\usepackage[capitalize,noabbrev]{cleveref}
\usepackage{multirow}
\usepackage{booktabs} 

\theoremstyle{plain}

\theoremstyle{definition}
\newtheorem{definition}{Definition}[section]

\theoremstyle{remark}

\usepackage{changepage}
\definecolor{myhighlight}{RGB}{220,240,255}
\usepackage{multicol}
\SetKwInput{KwInput}{Input}
\SetKwInput{KwOutput}{Output}
\usepackage{tcolorbox}
\usepackage{caption}
\usepackage{thm-restate}
\usepackage{wrapfig}
\usepackage{float}

\tcbset{
  takeaway/.style={
    colback=gray!10!white,
    colframe=gray!60!black,
    fonttitle=\bfseries,
    coltitle=white,
    boxrule=1pt,
    arc=4pt,
    left=6pt,
    right=6pt,
    top=6pt,
    bottom=6pt,
    title=Takeaway,
  }
}
\tcbset{
  example/.style={
    colback=blue!10!white,  
    colframe=blue!60!black, 
    fonttitle=\bfseries,
    coltitle=white,
    boxrule=1pt,
    arc=4pt,
    left=6pt,
    right=6pt,
    top=6pt,
    bottom=6pt,
    title=Example,
  }
}

\begin{document}

\title{Can Prompt Difficulty be Online Predicted for Accelerating RL Finetuning of Reasoning Models?}


\author{Yun Qu}
\orcid{0009-0000-1803-8435}
\affiliation{%
  \institution{Tsinghua University}
  \city{Beijing}
  \country{China}
}

\author{Qi Wang}
\authornote{Correspondence: cheemswang@mail.tsinghua.edu.cn, xyji@tsinghua.edu.cn}
\orcid{0009-0003-7758-725X}
\affiliation{%
  \institution{Tsinghua University}
  \city{Beijing}
  \country{China}
}

\author{Yixiu Mao}
\orcid{0009-0000-7302-5039}
\affiliation{%
  \institution{Tsinghua University}
  \city{Beijing}
  \country{China}
}

\author{Vincent Tao Hu}
\orcid{0000-0003-1561-3216}
\affiliation{%
  \institution{CompVis @ LMU Munich, Munich Center for Machine Learning (MCML)}
  \city{Munich}
  \country{Germany}
}

\author{Bj\"{o}rn Ommer}
\orcid{0000-0003-0766-120X}
\affiliation{%
  \institution{CompVis @ LMU Munich, Munich Center for Machine Learning (MCML)}
  \city{Munich}
  \country{Germany}
}

\author{Xiangyang Ji}
\authornotemark[1]
\orcid{0000-0001-9542-5260}
\affiliation{%
  \institution{Tsinghua University}
  \city{Beijing}
  \country{China}
}

\renewcommand{\shortauthors}{Qu et al.}


\begin{abstract}
Recent advances have witnessed the effectiveness of reinforcement learning (RL) finetuning in enhancing the reasoning capabilities of large language models (LLMs). 
The optimization process often requires numerous iterations to achieve satisfactory performance, resulting in high computational costs due to the need for frequent prompt evaluations under intensive LLM interactions and repeated policy updates.
Appropriate online prompt selection methods reduce iteration steps by prioritizing informative prompts during training, while the pipeline's reliance on exhaustive prompt evaluation and subset selection for optimization still incurs substantial computational overhead due to frequent LLM inference calls.
Distinguished from these direct evaluate-then-select schemes, this work investigates iterative approximate evaluation for arbitrary prompts and introduces Model Predictive Prompt Selection (MoPPS), a Bayesian risk-predictive framework that online estimates prompt difficulty without requiring costly LLM interactions.
Technically, MoPPS models each prompt's success rate as a latent variable, performs streaming Bayesian inference, and employs posterior sampling in a constructed multi-armed bandit machine, enabling efficient and adaptive prompt selection.
Extensive experiments across mathematics, planning, and vision-based geometry tasks show that MoPPS reliably predicts prompt difficulty and accelerates training with significantly reduced LLM rollouts.
Our code is available at \url{https://github.com/thu-rllab/MoPPS}.
\end{abstract}




\maketitle

\section{Introduction}\label{sec:intro}

Reinforcement learning (RL) finetuning has become a prominent method for enhancing capabilities of large language models (LLMs)  \citep{guo2025deepseek, team2025kimi,jaech2024openai,huang2025foundation}, leading to notable reasoning improvements in the presence of complicated tasks such as mathematical problem solving~\citep{luo2025deepscaler} and code generation~\citep{luo2025deepcoder}.
Despite its effectiveness, RL finetuning of LLMs is widely known to be expensive in computations and memory usage during inference calls, as it requires intensive rollouts for policy evaluation and updates in LLMs~\citep{zheng2025act, lin2025cppo}.

\textbf{Online Prompt Selection Matters in RL Finetuning:} 
In RL finetuning of LLMs, random sampling from the prompt dataset is common for chain-of-thought generation and policy optimization. 
However, it often fails to capture informative prompts and suffers from inefficiency and redundancy, while token generation itself is resource-intensive~\citep{zheng2025act}.
Recent work emphasizes data quality~\citep{guo2025deepseek,grattafiori2024llama3} and explores online prompt selection~\citep{cui2025process,yang2024qwen2math,yu2025dapo,bae2025online,meng2025mm,zhang2025srpo,xiong2025minimalist}, where training batches are curated by prioritizing prompts based on quality or difficulty~\citep{cui2025process,yu2025dapo,meng2025mm,bae2025online}.
While these methods improve performance and even accelerate training, evaluating prompt difficulty across large candidate pools still incurs heavy computational overhead~\citep{chen2025self}.

\setlength{\columnsep}{10pt}
\begin{wrapfigure}{r}{0.25\textwidth}
\centering
\vspace{-15pt}
\includegraphics[width=1.0\linewidth]{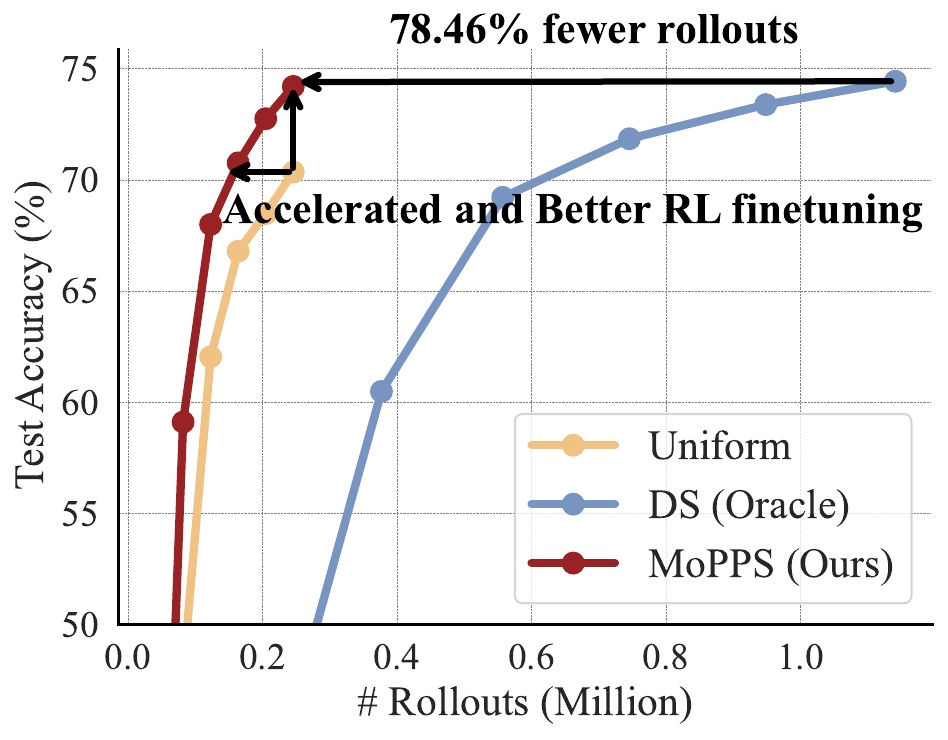}
\vspace{-23pt}
\caption{Performance and efficiency of prompt selection on Countdown. MoPPS outperforms uniform selection in training efficiency and performance and reduces rollouts by 78\% compared to DS \citep{yu2025dapo}.}
\vspace{-15pt}
\label{fig: fig1}
\end{wrapfigure}
\textbf{Promise and Challenge in Amortizing Prompt Evaluation:}
A promising alternative to alleviate the above predicament, i.e., the expensive cost of prompt evaluation under LLMs during online selection, is model predictive task sampling (MPTS) \citep{wang2025model}, where a lightweight risk predictive model is employed to estimate the expected utility, e.g., the returns of agent-environment interactions over iterations.
Such a framework reuses the optimization history as the selection prior, amortizes the process of exact policy evaluation and achieves efficient robust adaptation in meta reinforcement learning and domain randomization scenarios~\citep{qu2025fast}.
Meanwhile, it can be seamlessly integrated with diverse data selection heuristics.
Nevertheless, scaling vanilla MPTS to LLM finetuning is nontrivial: (i) there are no explicit identifiers, e.g., continuous real-vector, to construct the risk predictive model since the prompt dataset is typically finite and in the form of language tokens; (ii) the variable of interest for active prompt selection is the success rate of the reasoning problem, whose distribution is difficult to depict and dynamically evolves with LLMs' updates.

Realizing the importance of online scoring prompt's difficulty for effectively RL finetuning LLMs and considering the above-mentioned challenges, this work aims to answer two \textbf{research questions (RQs)} below:
\begin{enumerate}
    \item Can prompt difficulty be dynamically predicted without exactly interacting with LLMs?
    \item How can predicted outcome serve active data sampling for enhancing LLMs' reasoning power?
\end{enumerate}

\textbf{Approximate Inference towards Prompt Difficulty for Active Selection:}
In response to these RQs, this work develops the Model Predictive Prompt Selection (MoPPS) method for online scoring the prompt difficulty approximately, which is simple to implement yet significantly improves learning efficiency in RL finetuning.
Here, we formulate online prompt selection as a sequential decision-making problem and solve it with a dynamic Bernoulli bandit-based data mining strategy~\citep{berry1972bernoulli, russo2014learning}. 
In other words, each prompt is treated as an arm with stochastic binary rewards drawn from a latent variable as the success rate, and then we adopt posterior sampling to screen prompts in a streaming way.
The sampling outcome of latent variables avoids exact evaluation of prompts, facilitates exploration as stochastic optimism, and supports informative prompt selection without extra LLM inference.

\textbf{Contributions and Primary Findings:}
This work adopts a predict-then-optimize principle and successfully applies the concept of MPTS to the practice of RL finetuning LLMs.
The primary contributions are three-fold:
\begin{enumerate}
    \item We present a probabilistic graphical model to characterize RL finetuning LLMs, where the success rate works as the latent variable.
    The Bernoulli bandit machine is then introduced to enable online prompt selection, offering a new scheme for designing flexible active selection strategies.
    \item We constitute the principled posterior update method to efficiently estimate prompt difficulty with theoretical guarantee, which surrogates evaluation cost during high-quality prompt selection.
    \item Our framework is easy to implement and can be seamlessly integrated into a range of RL algorithms with various LLM backbones, providing a system-efficient component that complements the pipeline of active RL finetuning.
\end{enumerate}

Extensive experiments on complicated reasoning tasks spanning mathematics, planning, and vision-based geometry positively answer two RQs and reveal that MoPPS reliably predicts prompt difficulty, exhibiting high correlations with ground-truth evaluation.
Benefiting such predictability, MoPPS significantly accelerates RL finetuning, e.g., achieving $\bm{1.8\times}$ speedup over uniform sampling on Countdown~\citep{tinyzero}, and yields better performance, with over $\bm{24.4\%}$ relative improvements on the AIME24 while training on MATH~\citep{hendrycksmath2021}.

Importantly, our method achieves comparable performance of evaluation-intensive methods like dynamic sampling~\citep{yu2025dapo} with only $\bm{21\%}$ rollouts, significantly reducing computational costs.

\section{Preliminary}

\subsection{Notations}
The prompt $\tau$ in reasoning tasks can be in the form of a mathematical or logical reasoning problem, e.g., ``What is the degree of the polynomial $(4 +5x^3 +100 +2\pi x^4 + \sqrt{10}x^4 +9)$?" in the MATH~\citep{hendrycksmath2021} dataset.
Let $\mathcal{T} = \{\tau_i\}_{i=1}^N$ denote the full pool of prompts, where each $\tau_i$ represents a unique prompt. 
We define the parameter of the LLM at $t$-th training step by $\pi_{\bm \theta_t}$.
The selected prompt batch at $t$-th training step is $\mathcal{T}_t^{\mathcal{B}}=\{\tau_{t,i}\}_{i=1}^{\mathcal{B}}\subset\mathcal{T}$ with $\mathcal{B}$ the batch size.

At $t$-th time step, the LLM produces $k$ independent responses $\bm y^t_\tau = \{y_\tau^{t,j}\}_{j=1}^k$ conditioned on a prompt $\tau$, where each $y_\tau^{t,i}$ is sampled in an auto-regressive manner.
Here, we associate each prompt $\tau$ with a success rate $\gamma_\tau^t \in [0,1]$ and treat it as the latent variable, which reflects the chance of $\tau$'s problem-solving success under the current policy.
The set of success rates for the prompt batch is denoted by $\Gamma_t^{\mathcal{B}} = \{\gamma^t_{\tau_{t,i}}\}_{i=1}^{\mathcal{B}}$.
Then, each response is scored via examining the ground-truth answer, leading to a binary reward function:
\begin{equation}
\small
r_\tau^{t,j} \sim \mathrm{Bernoulli}(\gamma^t_\tau), \ 
r_\tau^{t,j} = 
\begin{cases}
1, & \text{if response $j$ is correct}, \\
0, & \text{otherwise},
\end{cases}
\ j = 1, \dots, k.
\nonumber
\end{equation}
For each prompt $\tau$, $\bm r^t_{\tau} = \{r_{\tau}^{t,i}\}_{i=1}^k$ denotes the set of rewards for each $k$ generated responses.
And the feedback collected for the prompt batch at step $t$ is written as $\mathcal{R}_t^{\mathcal{B}} = \{\bm r^t_{\tau_{t,i}}\}_{i=1}^{\mathcal{B}}$.
Hence, the likelihood of observing $\bm{r}^t_\tau$, i.e., success counts, given $\gamma^t_\tau$ is binomial:
\begin{equation}
\begin{aligned}
&p(r_{\tau}^{t,i})=(\gamma^t_\tau)^{[r_{\tau}^{t,i}=1]}\cdot(1-\gamma^t_\tau)^{[r_{\tau}^{t,i}=0]}\\
&\Rightarrow
p(\bm r^t_\tau \mid \gamma^t_\tau)
= \binom{k}{s^t_\tau}\cdot(\gamma^t_\tau)^{s^t_\tau}\cdot(1 - \gamma^t_\tau)^{k - s^t_\tau} \
\text{with}
\
s^t_\tau \triangleq \sum_{j=1}^k r_\tau^{t,j}.
\end{aligned}
\label{eq:likelihood}
\end{equation}
For simplicity, this work focuses on binary reward signals in RL finetuning.
However, the proposed method readily applies to richer reward forms such as format rewards~\citep{tinyzero}, either by modeling them directly or by binarizing them through thresholding or rounding.

Finally, we write the entire optimization history up to step $t$ as $H_t = \{\mathcal{T}_i^{\mathcal{B}}, \mathcal{R}_i^{\mathcal{B}}\}_{i=0}^t$, which records all selected batches and their corresponding feedback over iteration.

\subsection{RL Finetuning for LLM }\label{sec:rlfinetune}

The objective of RL finetuning is to optimize the LLM parameters $\bm\theta$ to maximize the expected reward over the prompt distribution. 
In mathematics, this corresponds to
\begin{equation}
    \max_{\bm\theta} \; \mathbb{E}_{\tau \sim \mathcal{T}, \; y \sim \pi_{\bm\theta}(\cdot|\tau)} \left[ r(\tau, y) \right],
\end{equation}
where $\pi_{\bm\theta}(y|\tau)$ denotes the model's conditional distribution over responses given a prompt $\tau$, and $r(\tau, y)$ is a reward function evaluating the quality of response $y$ under prompt $\tau$.

\paragraph{Proximal Policy Optimization (PPO)}
PPO~\citep{schulman2017proximal} is a widely adopted RL algorithm for finetuning LLMs~\citep{hu2025open, zeng7b}. It enhances training stability by enforcing a trust-region constraint, ensuring policy updates remain close to the previous policy $\pi_{\bm\theta_{\text{old}}}$ via a clipped surrogate objective:
\begin{equation}
\small
\begin{aligned}
    \mathcal{J}_{\text{PPO}}(\bm\theta) &= 
    \mathbb{E}_{\tau \sim \mathcal{T}_t^{\mathcal{B}}, \;  y_{\leq t} \sim \pi_{\bm\theta_{\text{old}}}(\cdot|\tau)} \\
    &\left[ \min \left( \rho_{t}(\bm\theta) \cdot \hat{A}_{t}, \;
    \text{clip}(\rho_{t}(\bm\theta), 1 - \epsilon, 1 + \epsilon) \cdot \hat{A}_{t} \right) \right],
\end{aligned}
\end{equation}
where $y_{<t}$ and $y_t$ denote the generated token prefix and the current token at position $t$, respectively, $\rho_{t}(\bm\theta) = \frac{\pi_{\bm\theta}(y_t|\tau,y_{<t})}{\pi_{\bm\theta_{\text{old}}}(y_t|\tau,y_{<t})}$ is the importance sampling ratio with $\epsilon$ the clipping range.
The estimated advantage $\hat{A}_t$ is computed using the generalized advantage estimation~\citep{schulman2015high}.

\paragraph{Group Relative Policy Optimization (GRPO)}
GRPO~\citep{shao2024deepseekmath} estimates the advantage in a group-normalized manner and eliminates the need for the value function.
For each prompt $\tau \in \mathcal{T}_t^{\mathcal{B}}$, the model generates $k$ rollouts $\{y_\tau^i\}_{i=1}^k$ from the old policy $\pi_{\bm\theta_{\text{old}}}$.
Then, the objective of GRPO is written as:
\begin{equation}
\small
\begin{aligned}
   &\mathcal{J}_{\text{GRPO}}(\bm\theta) = 
    \mathbb{E}_{\substack{\tau \sim \mathcal{T}_t^{\mathcal{B}}, \; \{y_\tau^i\}_{i=1}^k \sim \pi_{\bm\theta_{\text{old}}}(\cdot|\tau)}} 
    \left[ \frac{1}{k}\sum_{i=1}^k\frac{1}{|y_\tau^i|}\sum_{t=1}^{|y_\tau^i|}
    \left(
    \min\left(  \right. \right. \right. \\
    & \left. \left. \left. \rho_{i,t}(\bm\theta) \cdot \hat{A}_{i}, \;
    \text{clip}(\rho_{i,t}(\bm\theta), 1 - \epsilon, 1 + \epsilon) \cdot \hat{A}_{i} \right) 
    - \beta D_{KL}(\pi_{\bm \theta}||\pi_{\text{ref}})
    \right) \right]
\end{aligned}
\end{equation}
where $\rho_{i,t}(\bm\theta) = \frac{\pi_{\bm\theta}(y^i_t|\tau,y^i_{<t})}{\pi_{\bm\theta_{\text{old}}}(y^i_t|\tau,y^i_{<t})}$ and $\pi_{\text{ref}}$ is a fixed reference policy.
The KL divergence term penalizes deviation from $\pi_{\text{ref}}$, with $\beta$ controlling the regularization strength, and the group-relative advantage for the $i$-th response is calculated via normalizing $\{r_\tau^i\}_{i=1}^k$:
\begin{equation}
    \hat{A}_{i} = \frac{r_\tau^i-\text{mean}(\{r_\tau^i\}_{i=1}^k)}{\text{std}(\{r_\tau^i\}_{i=1}^k}.
\end{equation}

\paragraph{Online Prompt Selection}
RL finetuning of LLMs typically suffers from substantial computational overhead, facilitating a line of work~\citep{yu2025dapo, zhang2025srpo, chen2025self, bae2025online} to explore online prompt selection for the purpose of training acceleration.

One recent SOTA approach is \textbf{Dynamic Sampling (DS)} developed in~\citep{yu2025dapo}, which is driven by the observation that algorithms such as GRPO encounter vanishing gradients when prompts have success rate equal to $0$ or $1$. 
To mitigate this, DS over-samples a larger candidate set $\mathcal{T}_t^{\hat{\mathcal{B}}} \subseteq \mathcal{T}$ with $\hat{\mathcal{B}} > \mathcal{B}$, then filters out uninformative prompts to construct the actual training batch:
\begin{equation}
    \mathcal{T}_t^{\mathcal{B}} = \left\{ \tau \in \mathcal{T}_t^{\hat{\mathcal{B}}} \;\middle|\; 0 < \sum_{i=1}^k r^i_\tau < k \;\text{or}\; \text{std}\left(\{r^i_\tau\}_{i=1}^k\right) > 0 \right\}.
    \label{eq:ds}
\end{equation}
Similar ideas, which prioritize prompts with success rates near $0.5$, are proposed in~\citep{bae2025online, chen2025self} and show that the optimization process benefits from such a configuration.
These online prompt selection methods increase the proportion of effective prompts in each batch, thereby reducing the number of iteration steps.
However, the reduced training steps come at the expense of additional computational cost from exact LLM evaluations~\citep{zheng2025act}.

\subsection{Model Predictive Task Sampling}

MPTS~\citep{wang2025model} amortizes costly policy or data evaluation through active inference by modeling task optimization as a generative process. 
Using streaming variational inference~\citep{broderick2013streaming, nguyen2017variational}, it builds a predictive model $p(\ell \vert \tau, H_t; \bm\theta_t)$ to estimate evaluation metrics like returns in RL or training loss. 
These predictions guide active sampling via criteria such as UCB~\citep{auer2002finite} or posterior sampling~\citep{qu2025fast}, reducing costly environment interactions or inference calls.

MPTS has shown its computational and annotation efficiency in adaptive decision-making~\citep{qu2025fast} and supervised finetuning~\citep{wang2025model}, but assumes continuous task spaces with explicit identifiers. 
In particular, RL finetuning of LLMs involves discrete, prompt-defined tasks and emphasizes training efficiency. 
This work instead targets (i) amortizing prompt evaluation costs and (ii) identifying acquisition strategies to accelerate RL finetuning.

\section{Method}

This section surrounds two \textbf{RQs} from Sec.~\ref{sec:intro} and presents a principled framework for accelerating RL finetuning via model predictive prompt selection. 
We recast RL finetuning as a generative process with success rate as a latent variable, estimate prompt-specific performance using Bayesian Bernoulli bandits, and adopt Thompson sampling with selection criteria for efficient RL finetuning.
\begin{figure}[ht]
    \centering
    \includegraphics[width=\linewidth]{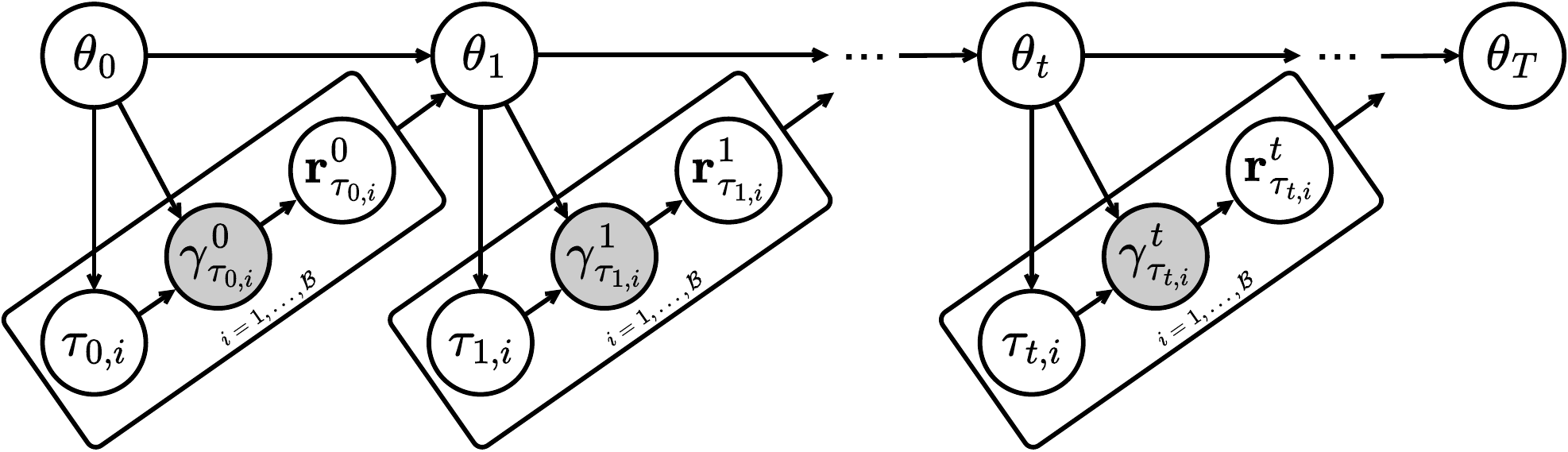}
    \vspace{-15pt}
    \caption{Probabilistic graphical model for RL finetuning of LLMs.
    The reward signal $\bm{r}^t_{\tau_{t,i}}$ is a set of binary values evaluating the $k$ generated responses, governed by the latent success rate $\gamma^t_{\tau_{t,i}}$.
    The prompt batch $\{\tau_{t,i}\}_{i=1}^{\mathcal{B}}$ is selected under specific criteria based on current LLM $\bm\theta_t$.
    The white and grey nodes respectively denote observed and latent variables.}
    \vspace{-10pt}
    \label{fig:pgm}
\end{figure}

\subsection{RL Finetuning as A Generative Process}

The process of RL finetuning involves a couple of variables, e.g., LLMs' parameters $\bm\theta_t$, the prompt batch $\mathcal{T}_t^{\mathcal{B}}$, the generated responses, and the batch reward signals $\mathcal{R}_t^{\mathcal{B}}$ over iterations.
Recent advances \citep{yu2025dapo} have demonstrated the importance of prompt selection based on specific criteria in accelerating the training process.

\paragraph{Generative Process of Active RL Finetuning.}
Putting these ingredients together, we can express the joint distribution of relevant variables and derive its factorization as: 
\begin{equation}\label{eq:full-joint}
\begin{aligned}
  &p\bigl(\bm\theta_{0:T},\mathcal{T}_{0:T-1}^{\mathcal{B}},\mathcal{R}_{0:T-1}^{\mathcal{B}}\bigr)\quad= \quad p(\bm\theta_0)
    \prod_{t=0}^{T-1}
      \;\underbrace{p\bigl(\mathcal{T}_t^{\mathcal{B}}\mid \bm\theta_t \bigr)}_{\text{\textbf{Prompt Selection}}}
      \;\\
  &\underbrace{p\bigl(\bm\theta_{t+1}\mid \bm\theta_t, \mathcal{R}_t^{\mathcal{B}}, \mathcal{T}_t^{\mathcal{B}}\bigr)}_{\text{\textbf{Policy Optimization}}}
      \underbrace{\int p\bigl(\Gamma_t^{\mathcal{B}}\mid \mathcal{T}_t^{\mathcal{B}}, \bm\theta_t\bigr)
      \;p\bigl(\mathcal{R}_t^{\mathcal{B}}\mid \Gamma_t^{\mathcal{B}}\bigr)\;d \Gamma_{t}^{\mathcal{B}}}_{\textbf{Prompt Evaluation}},
\end{aligned}
\end{equation}
where the prompt selection term encompasses some selection mechanism and the prompt evaluation is associated with a collection of latent variables as the success rate in Fig.~\ref{fig:pgm}.

The prompt evaluation term in Eq.~(\ref{eq:full-joint}) implies that response generation requires several LLM inferences, which is compute-intensive but crucial to optimize policy, and also used to assess prompt difficulty for online selection, as discussed below.
When no prompt selection criteria are incorporated in optimization, random prompt selection, such as $\text{Uniform}(\mathcal{T}_t^\mathcal{B})$, is independent of the updated policy $\bm\theta_t$ and incurs no extra inference overhead.
However, random selection suffers from sampling redundancies and tends to consume numerous iterations to converge~\citep{yu2025dapo}.

\paragraph{Price to Pay in Prompt Evaluation and Selection.}
While online prompt selection improves sample efficiency, it often incurs substantial computational cost, as it typically requires additional real evaluations on a larger candidate set $\mathcal{T}_t^{\hat{\mathcal{B}}}$ ($\hat{\mathcal{B}} > \mathcal{B}$) to score and filter prompts~\citep{yu2025dapo, bae2025online}:
\begin{equation}\label{eq:prompt_select_pipeline}
\begin{aligned}
\text{Online Prompt } &\text{Selection:}\\ 
&\mathcal{T}_t^{\hat{\mathcal{B}}} \xrightarrow{\text{\textbf{Evaluate}}} \{\mathcal{T}_t^{\hat{\mathcal{B}}},\mathcal{R}_t^{\hat{\mathcal{B}}} \}\xrightarrow{\text{\textbf{Filter}}} \{\mathcal{T}_t^{\mathcal{B}},\mathcal{R}_t^{\mathcal{B}} \}.
\end{aligned}
\end{equation}
Formally, the conditional distribution of prompt selection can be expressed as:
\begin{equation}\label{eq:prompt_select_integral}
\begin{aligned}
p(\mathcal{T}_t^\mathcal{B}\mid\bm \theta_t) 
&= \int
    p(\mathcal{T}_t^\mathcal{B} \mid \mathcal{R}_t^{\hat{\mathcal{B}}}, \mathcal{T}_t^{\hat{\mathcal{B}}}) p(\mathcal{T}_t^{\hat{\mathcal{B}}}) 
    \\
    &\underbrace{\int p(\Gamma_t^{\hat{\mathcal{B}}} \mid \mathcal{T}_t^{\hat{\mathcal{B}}}, \bm\theta_t)
    \, p(\mathcal{R}_t^{\hat{\mathcal{B}}} \mid \Gamma_t^{\hat{\mathcal{B}}}) 
    \, d\Gamma_t^{\hat{\mathcal{B}}}}_{\text{\textbf{Extra Prompt Evaluation}}} \;d\mathcal{R}_t^{\hat{\mathcal{B}}}\;d\mathcal{T}_t^{\hat{\mathcal{B}}},
\end{aligned}
\end{equation}
where $p(\mathcal{T}_t^{\hat{\mathcal{B}}})$ denotes the probability of sampling a larger candidate set, and $p(\mathcal{T}_t^\mathcal{B} \mid \mathcal{R}_t^{\hat{\mathcal{B}}}, \mathcal{T}_t^{\hat{\mathcal{B}}})$ specifies the conditional probability of selecting the prompt batch after extra prompt evaluation under some criteria.

As can be seen in Eq.~(\ref{eq:prompt_select_pipeline}) and (\ref{eq:prompt_select_integral}),
though this explicit evaluate-then-filter pipeline online identifies crucial prompts and accelerates learning, the additional inference over the candidate batch substantially brings computational and memory burden per-step cost.

\subsection{Bayesian Inference towards Prompt Success Rate}\label{sec:bayesian}

To circumvent additional evaluation overhead, we draw inspiration from MPTS~\citep{wang2025model} and introduce a Bayesian surrogate model to (i) dynamically model the success rate $\gamma^t_\tau$ for each prompt using optimization histories and (ii) enable posterior-guided sampling of informative prompts without requiring additional LLM inference.

\paragraph{Exploitation and Exploration in Prompt Selection.}
Prompt selection requires sequentially choosing prompts with unknown effectiveness that must be dynamically estimated from binary success feedback.
To balance exploiting prompts with demonstrated effectiveness and exploring uncertain prompts that may provide more informative learning signals, we formulate online prompt selection as a stochastic Bernoulli bandit problem.

\begin{definition}[\textbf{Prompt Selection Bernoulli Bandit}]
Each prompt $\tau \in \mathcal{T}$ is treated as an arm in a stochastic multi-armed bandit, characterized by an unknown success rate $\gamma^t_\tau \in [0,1]$.
Pulling an arm corresponds to querying the current policy $\pi_{\bm\theta_t}$ on prompt $\tau$ and observing binary feedback $r^t_\tau \in \{0, 1\}$ indicating success or failure.
The objective is not to maximize cumulative reward but to preferentially select prompts that provide the most informative gradients for model learning, e.g., $\gamma^t_\tau \approx 0.5$~\citep{bae2025online, chen2025self}.
\end{definition}

This formulation offers a unified framework for analyzing prompt selection strategies in RL finetuning and supports principled algorithm design based on bandit theory.
Prior methods based on real evaluation and deterministic filtering can be seen as a special case of this framework, corresponding to greedy exploitation with near-complete candidate feedback.
In contrast, we introduce a Bayesian model that maintains and updates a posterior belief over each $\gamma^t_\tau$, enabling efficient prompt selection that naturally balances exploration and exploitation without costly LLM inference.

\begin{figure*}[t]
    \centering
    \vspace{-5pt}
    \includegraphics[width=0.9\linewidth]{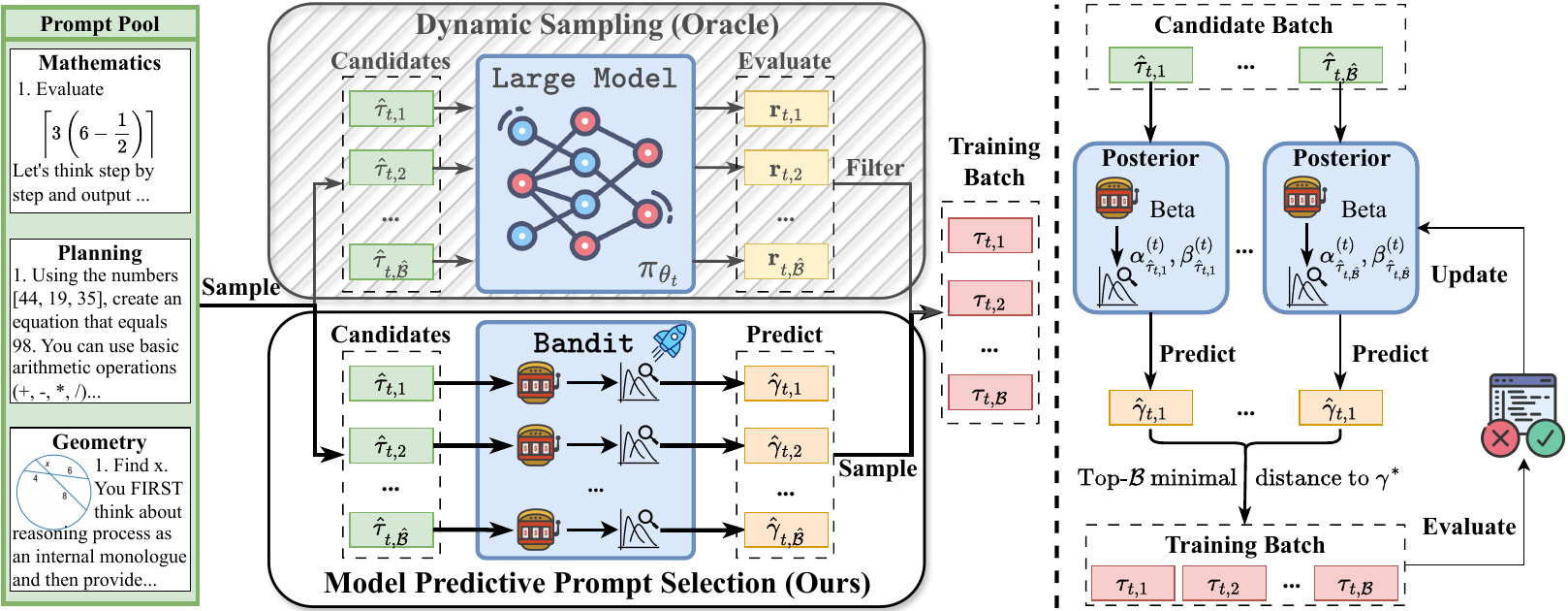}
    \vspace{-5pt}
    \caption{Framework Overview. \textbf{Left:} Comparison between \emph{Dynamic Sampling (Oracle)}, which filters prompts based on actual LLM evaluation on candidates, and our \emph{Model Predictive Prompt Selection (MoPPS)}, which predicts success rates to avoid extra inference cost. \textbf{Right:} MoPPS  predicts success rates for candidates from posterior parameters, based on which prompts closest to a target $\gamma^*$ are selected for training; the posterior is then updated using new feedback.}
    \Description{Framework Overview. \textbf{Left:} Comparison between \emph{Dynamic Sampling (Oracle)}, which filters prompts based on actual LLM evaluation on candidates, and our \emph{Model Predictive Prompt Selection (MoPPS)}, which predicts success rates to avoid extra inference cost. \textbf{Right:} MoPPS  predicts success rates for candidates from posterior parameters, based on which prompts closest to a target $\gamma^*$ are selected for training; the posterior is then updated using new feedback.}
    \label{fig:framework}
    \vspace{-5pt}
\end{figure*}
\paragraph{Recursive Bayesian Update.} Next, we detail the recursive Bayesian update procedure for efficient posterior inference of the success rates $\gamma^t_\tau$.
To enable tractable inference and closed-form posterior updates, we place a Beta prior over the initial success rate:
\begin{equation}
\gamma^0_\tau \sim \mathrm{Beta}(\alpha_\tau^{0}, \beta_\tau^{0}),
\end{equation}
where $\alpha_\tau^{0}$ and $\beta_\tau^{0}$ reflect prior pseudo-counts of successes and failures, typically set to $(1,1)$ for a uniform prior.
Other informative priors can also be incorporated, as evaluated in Sec.~\ref{sec:selection}.

By Bayes rule, the posterior distribution over $\gamma^t_\tau$ given observations up to step $t$ is:
\begin{equation}
\underbrace{p(\gamma^t_\tau \mid H_t)}_{\text{\textbf{Updated Posterior}}} \propto \quad\underbrace{p(\bm r^t_\tau \mid \gamma^t_\tau)}_{\text{\textbf{Likelihood}}} \cdot \underbrace{p(\gamma^t_\tau \mid H_{t-1})}_{\text{\textbf{Conjugate Prior}}},
\label{eq:bayes}
\end{equation}
where $p(\gamma^t_{\tau} \mid H_{t-1})\sim \mathrm{Beta}(\alpha^{t}_\tau, \beta^{t}_\tau)$ represents the conditional prior using the last time updated posterior $p(\gamma^{t-1}_{\tau} \mid H_{t-1})$ as the proxy when $t\ge 1$ and $p(\bm r^t_{\tau} \mid \gamma^t_{\tau})$ is the likelihood of observing the feedback under prompt $\tau$.

Since the Beta distribution is conjugate to the Bernoulli likelihood in Eq.~(\ref{eq:likelihood}), the posterior of $\gamma$ also follows a Beta distribution:
\begin{equation}
\gamma^t_{\tau} \mid H_{t} \sim \mathrm{Beta}(\alpha^{t'}_\tau, \beta^{t'}_\tau),
\end{equation}
with the following recursive update rules:
\begin{equation}
\begin{aligned}
\alpha^{t'}_{\tau} = \alpha^{t}_{\tau} + s^t_\tau, \quad
\beta^{t'}_{\tau} = \beta^{t}_{\tau} + k - s^t_\tau.
\end{aligned}
\end{equation}
These serve as the prior for the next step under the streaming Bayes setup:
\begin{equation}
\alpha^{t+1}_\tau = \alpha^{t'}_\tau, \quad \beta^{t+1}_\tau = \beta^{t'}_\tau.
\end{equation}
These updates accumulate evidence over time, with $\alpha^{t}_\tau$ and $\beta^{t}_\tau$ representing the total (pseudo) counts of observed successes and failures for prompt $\tau$, respectively, up to step $t$.
This posterior serves as a compact and efficient representation of uncertainty over prompt difficulty, supporting downstream sampling and decision-making without requiring LLM inference.

\paragraph{Incorporating Temporal Discounting.}  
Note that the distribution of $\gamma^t_\tau$ relies on the updated model parameters, and a significant update of $\bm\theta_{t}$ over iterations makes the distribution nonstationary.
To precisely estimate the distribution parameters under these scenarios, we apply exponential discounting to past observations, placing more weight on recent feedback.
With the decay factor $\lambda \in (0,1)$, we derive the update rule for the parameter posterior at step $t$ as:
\begin{equation}
\begin{aligned}
&\alpha^{t'}_{\tau} = \lambda \cdot \alpha^{t}_{\tau} + (1 - \lambda) \cdot \alpha_{\tau}^{0}  + s^t_\tau,\\
&\beta^{t'}_{\tau} = \lambda \cdot \beta^{t}_{\tau} + (1 - \lambda) \cdot \beta_{\tau}^{0}  + k - s^t_\tau.
\end{aligned}
\label{eq:betaupdate}
\end{equation}
Such a design strikes a balance between adaptivity and stability in dynamic training regimes.
A lower $\lambda$ value places more emphasis on recent feedback, which helps adapt to nonstationary training dynamics.
Conversely, when the training dynamics are nearly stationary, setting $\lambda$ closer to 1 improves performance by making better use of historical data.
An ablation study on this strategy is presented in Appendix~\ref{appsec:ablate}.

\paragraph{Guaranteed Posterior Estimation and Efficiency Enhancement.} 
We derive Theorem~\ref{theorem_bound} to analyze the estimation error bound of the posterior mean as an estimator of the underlying time-varying success rate $\gamma^t_\tau$.
The proof is provided in Appendix~\ref{appsec:proof}.

\begin{restatable}[Bounded Success Rate Estimation Error]{theorem}{posteriormeanbound}\label{theorem_bound}
Define the posterior mean estimate at step $t$ as $\bar{\gamma}^{t}_\tau := \frac{\alpha^{t'}_\tau}{\alpha^{t'}_\tau + \beta^{t'}_\tau}$, and assume the true success rate drifts slowly, i.e., $|\gamma^{t}_\tau - \gamma^{t-1}_\tau| \le \delta,\ \forall{t}$. Then, with probability at least $1 - 2\exp(-2k\eta^2)$, the estimation error satisfies the recurrence inequality:
\begin{equation}
\epsilon_t := |\bar{\gamma}^{t}_\tau-\gamma^{t}_\tau|< \lambda\cdot (\epsilon_{t-1} + \delta) + \frac{(1-\lambda)}{2} + \eta.
\nonumber
\end{equation}
\end{restatable}
With high probability, the estimation error can be bounded by the previous error $\epsilon_{t-1}$, the drift magnitude $\delta$, and the tolerance $\eta$ due to the finite sampling size $k$.
This result indicates that the posterior reflects a reliable and adaptive estimate of the true success rate, securing effective prompt selection without additional LLM calls.
Moreover, the recursive inequality highlights the role of the decay factor $\lambda$, which controls the relative importance of past versus recent feedback.

We further analyze the computational complexity of MoPPS and DS~\citep{yu2025dapo}.
DS repeatedly samples candidate prompts, queries LLM rollouts, and filters out those that do not satisfy a predefined constraint until reaching $\mathcal{B}$ selected prompts.
Let $p_{\text{keep}}$ denote the expected retention probability of each sampled prompt.
$C_{\text{LLM}}$ quantifies the expected cost per prompt for generating and evaluating $k$ LLM rollouts, and $C_{\text{pred}}$ measures the cost for posterior estimation per prompt.
Then, the expected time complexity for prompt selection and evaluation per step is: $\mathcal{O}\left( \lceil\frac{1}{p_{\text{keep}}}\rceil \cdot\mathcal{B} \cdot k \cdot C_{\text{LLM}} \right)$ for DS while $\mathcal{O}\left(\hat{\mathcal{B}} \cdot C_{\text{pred}} + \mathcal{B} \cdot k \cdot C_{\text{LLM}} \right) \approx \mathcal{O}\left( \mathcal{B} \cdot k \cdot C_{\text{LLM}} \right)$ for MoPPS.
Since typically $p_{\text{keep}} < 1$ and $C_{\text{pred}} \ll C_{\text{LLM}}$, MoPPS significantly reduces computational overhead compared to DS by avoiding repeated LLM inference for prompt selection.

\subsection{Model Predictive Prompt Selection}
The empirical results in Fig.~\ref{fig:corr} show that the posterior distribution's estimate of a prompt's success rate correlates strongly with the ground truth.
This provides a reliable foundation for using the posterior as an efficient proxy to evaluate prompt difficulty without querying the expensive LLM.
The below illuminates the pipeline of MoPPS, comprising two critical steps.

\paragraph{Fast Success Rate Estimates from Approximate Posteriors.}
Instead of relying on the posterior mean, we employ Thompson Sampling~\citep{thompson1933likelihood}, drawing a sample from the Beta posterior to incorporate stochastic optimism into the success rate estimate:
\begin{equation}
\hat{\gamma}^t_\tau \sim \mathrm{Beta}(\alpha_\tau^{t}, \beta_\tau^{t})
\quad\forall\tau\in\mathcal{T}_{t}^{\hat{\mathcal{B}}}.
\end{equation}
Note that this sampling uses the conditional prior $p(\gamma^t_\tau \mid H_{t-1})$, as defined in Eq.~(\ref{eq:bayes}), as a proxy for the posterior $p(\gamma^t_\tau \mid H_t)$, since prompt selection is performed before querying the LLM. This design enables efficient, inference-free prompt selection.

Importantly, we adopt Thompson Sampling in this work for its simplicity and natural exploration-exploitation trade-off, which inherently serves as an uncertainty-aware data curation mechanism. MoPPS can also be seamlessly combined with other acquisition strategies such as UCB~\citep{auer2002finite}.
In addition, our lightweight prediction allows us to extend $\mathcal{T}_{t}^{\hat{\mathcal{B}}}$ to the entire pool $\mathcal{T}$ at negligible cost, which significantly improves the exploration space and is infeasible for prior methods that rely on exact evaluation.  

\paragraph{Active Prompt Selection from the Predicted Outcome.}
Prior works~\citep{bae2025online, chen2025self} indicate that prompts with intermediate difficulty, i.e., success rates near a target value $\gamma^*$, typically around $0.5$, yield the most informative gradients for RL finetuning.
Viewed from the perspective of curriculum learning~\citep{bengio2009curriculum}, maintaining a target difficulty induces an implicit curriculum that gradually transitions from easier to harder prompts as the model improves.
Leveraging this insight, at each step, we construct the training batch $\mathcal{T}_t^\mathcal{B}$ by selecting the $\mathcal{B}$ prompts whose sampled success rates $\hat{\gamma}^t_\tau$ are closest to $\gamma^*$:
\begin{equation}
\label{eq:topk}
\mathcal{T}_t^\mathcal{B} = \operatorname{Top-}\mathcal{B}\left( \left\{ \tau \in \mathcal{T}_{t}^{\hat{\mathcal{B}}} \;\middle|\; -\left|| \hat{\gamma}^t_\tau - \gamma^* |\right|_2^2 \right\} \right),
\end{equation}
MoPPS can also be easily integrated with alternative selection strategies, as evaluated in Sec.~\ref{sec:selection}.

\subsection{Implementation Pipeline}

Eq.~(\ref{eq:pps}) abstracts the core idea of our method, where the blue-bold \textcolor{blue}{\textbf{Predict}} step emphasizes replacing costly real prompt evaluations with efficient posterior-based prediction of success rates.
\begin{equation}
\begin{aligned}
\text{Model Predic} &\text{tive Prompt Selection:} \\
&\mathcal{T}_t^{\hat{\mathcal{B}}} \xrightarrow{\text{\textbf{\textcolor{blue}{Predict}}}} \{\mathcal{T}_t^{\hat{\mathcal{B}}},\hat{\Gamma}_t^{\hat{\mathcal{B}}} \}\xrightarrow{\text{\textbf{Select}}} \{\mathcal{T}_t^{\mathcal{B}} \}.
\end{aligned}
\label{eq:pps}
\end{equation}
The framework overview is illustrated in Fig.~\ref{fig:framework}. MoPPS \textit{retains computational efficiency}, \textit{encourages exploration}, and preserves the ability to \textit{prioritize the most beneficial prompts for policy updates}.
Algorithm~\ref{algo} presents the proposed MoPPS, which can be seamlessly integrated with any RL finetuning algorithm.

\begin{algorithm}
\caption{Model Predictive Prompt Selection (MoPPS)}
\KwInput{
Prompt pool $\mathcal{T} = \{\tau_i\}_{i=1}^N$;
Prior Beta parameters $\alpha, \beta$;
Candidate batch size $\hat{\mathcal{B}}$; Selected batch size $\mathcal{B}$;
Target success rate $\gamma^*$; 
Decay factor $\lambda$;
Reasoning model $\pi_{\bm \theta_0}$ with parameters $\bm \theta_0$;
Total training steps $T$
}
\KwOutput{Finetuned model $\pi_{\bm \theta_T}$}

$\forall{\tau\in\mathcal{T}}$, initialize posterior parameters $(\alpha^{0}_{\tau}, \beta^{0}_{\tau}) \leftarrow (\alpha, \beta)$ \;

\For{$t = 0$ \KwTo $T-1$}{
    Randomly sample candidate set $\mathcal{T}_{t}^{\hat{\mathcal{B}}} = \{\hat{\tau}_{t,i}\}_{i=1}^{\hat{\mathcal{B}}}$ from $\mathcal{T}$\;

    \tcp{\textcolor{blue}{Difficulty Prediction}}
    \ForEach{$\hat{\tau}_{t,i} \in \mathcal{T}_{t}^{\hat{\mathcal{B}}}$}{
        Sample predicted difficulty $\hat{\gamma}^t_{\hat{\tau}_{t,i}} \sim \mathrm{Beta}(\alpha^t_{\hat{\tau}_{t,i}}, \beta^t_{\hat{\tau}_{t,i}})$ \;
   }

    \tcp{\textcolor{blue}{Active Prompt selection}}
    Select $\mathcal{T}_{t}^\mathcal{B} = \{\tau_{t,i}\}_{i=1}^{\mathcal{B}}$ as the $\mathcal{B}$ prompts from $\mathcal{T}_{t}^{\hat{\mathcal{B}}}$ via Eq.~(\ref{eq:topk}) \;

    \ForEach{$\tau_{t,i} \in \mathcal{T}_{t}^\mathcal{B}$}{
        Generate responses $\bm{y}_{\tau_{t,i}} = \{y^j_{\tau_{t,i}}\}_{j=1}^{k}$ using $\pi_{\bm \theta_t}$ and \;
        Compute corresponding rewards $\bm{r}_{\tau_{t,i}} = \{r^j_{\tau_{t,i}}\}_{j=1}^{k}$ (e.g., binary correctness scores) \;
   }

    Update $\bm \theta_t$ using $\{(\tau_{t,i}, \bm{y}_{\tau_{t,i}}, \bm{r}_{\tau_{t,i}})\}_{i=1}^{\mathcal{B}}$ with a suitable RL algorithm to obtain $\bm \theta_{t+1}$\;

    \tcp{\textcolor{blue}{Posterior Update}}
    \ForEach{$\tau \in \mathcal{T}_{t}^\mathcal{B}$}{
        Update $(\alpha_{\tau}^{t+1}, \beta_{\tau}^{t+1})$ via Eq.~(\ref{eq:betaupdate}) \;
   }
}
\label{algo}
\end{algorithm}

\section{Experiments}
This section conducts experiments to evaluate whether MoPPS can online predict prompt difficulty, accelerate RL finetuning, and improve performance. 
Additional analyses explore its flexibility across selection strategies, use of prior knowledge, need for posterior updates, and compatibility with various RL algorithms.

\subsection{Experimental Setup}

We evaluate MoPPS across three representative reasoning tasks: \textbf{mathematics}, \textbf{planning}, and \textbf{multi-modal geometry}.
To demonstrate its versatility, we adopt \textbf{diverse LLM backbones} with different sizes, including base LLMs, distilled variants, and multi-modal models.
For RL finetuning, we use the widely adopted GRPO algorithm built on verl~\citep{sheng2024hybridflow} framework, though MoPPS is compatible with other algorithms as shown in Sec.~\ref{exp:algocompatible}.
Test accuracy is reported as the average \texttt{pass@1} over 16 independent generations per problem, computed on training curves and evaluation results.
Further implementation details are in Appendix~\ref{appsec:implementation}, along with additional experimental results like ablation studies in Appendix~\ref{appsec:results} and data examples in Appendix~\ref{appsec:dataexample}.

\paragraph{Mathematics Task} We train LLMs on the training split of the MATH dataset~\citep{hendrycksmath2021}, which consists of problems from mathematics competitions. Following prior work~\citep{luo2025deepscaler}, we use the DeepSeek-R1 distillation models R1-Distill-Qwen-1.5B and R1-Distill-Qwen-7B~\citep{guo2025deepseek}, and track performance on AIME24 during training. Final evaluations are conducted on benchmarks including AIME24, AMC23, MATH500~\citep{lightman2023let}, Minerva Math (Minerva.)~\citep{lewkowycz2022solving}, and OlympiadBench (Olympiad.)~\citep{he2024olympiadbench}.
\paragraph{Planning Task} We adopt the Countdown Number Game, which requires combining given numbers using basic arithmetic operations to reach a target value. Training is performed on a subset of the Countdown-34 (CD-34) dataset~\citep{tinyzero}, with performance tracked on a held-out split. Final evaluation is conducted on both CD-34 and a more challenging variant, Countdown-4 (CD-4). Following~\cite{chen2025self}, we use two base models: Qwen2.5-3B and Qwen2.5-7B~\citep{yang2024qwen2}.
\paragraph{Visual Geometry Task} Geometry problems require both visual understanding and reasoning. Two vision language models, Qwen2.5-VL-3B-Instruct and Qwen2.5-VL-7B-Instruct~\citep{bai2025qwen2}, are trained on the training split of the Geometry3k dataset~\citep{lu2021inter, geometry3k_dataset} and evaluated on its test split.

\paragraph{Baselines.}
Two common sampling strategies are compared with MoPPS: (1) \textbf{Uniform}, which samples prompts uniformly from the prompt pool; (2) \textbf{History Resampling (HR)}~\citep{zhang2025srpo}, which excludes fully solvable prompts each epoch; and (3) \textbf{Dynamic Sampling (DS)}~\citep{yu2025dapo}, which oversamples prompts and filters out uninformative ones based on their exact evaluation, as described in Eq.~(\ref{eq:ds}).  
Notably, \textbf{DS serves as an Oracle and latest SOTA} baseline since it relies on real evaluation feedback.
Our focus is on reducing computational overhead relative to DS rather than outperforming it in overall accuracy.

\subsection{Main Results}

\subsubsection{Highly Correlated Difficulty Prediction}

\begin{figure*}[t]
    \centering
    \includegraphics[width=0.75\linewidth]{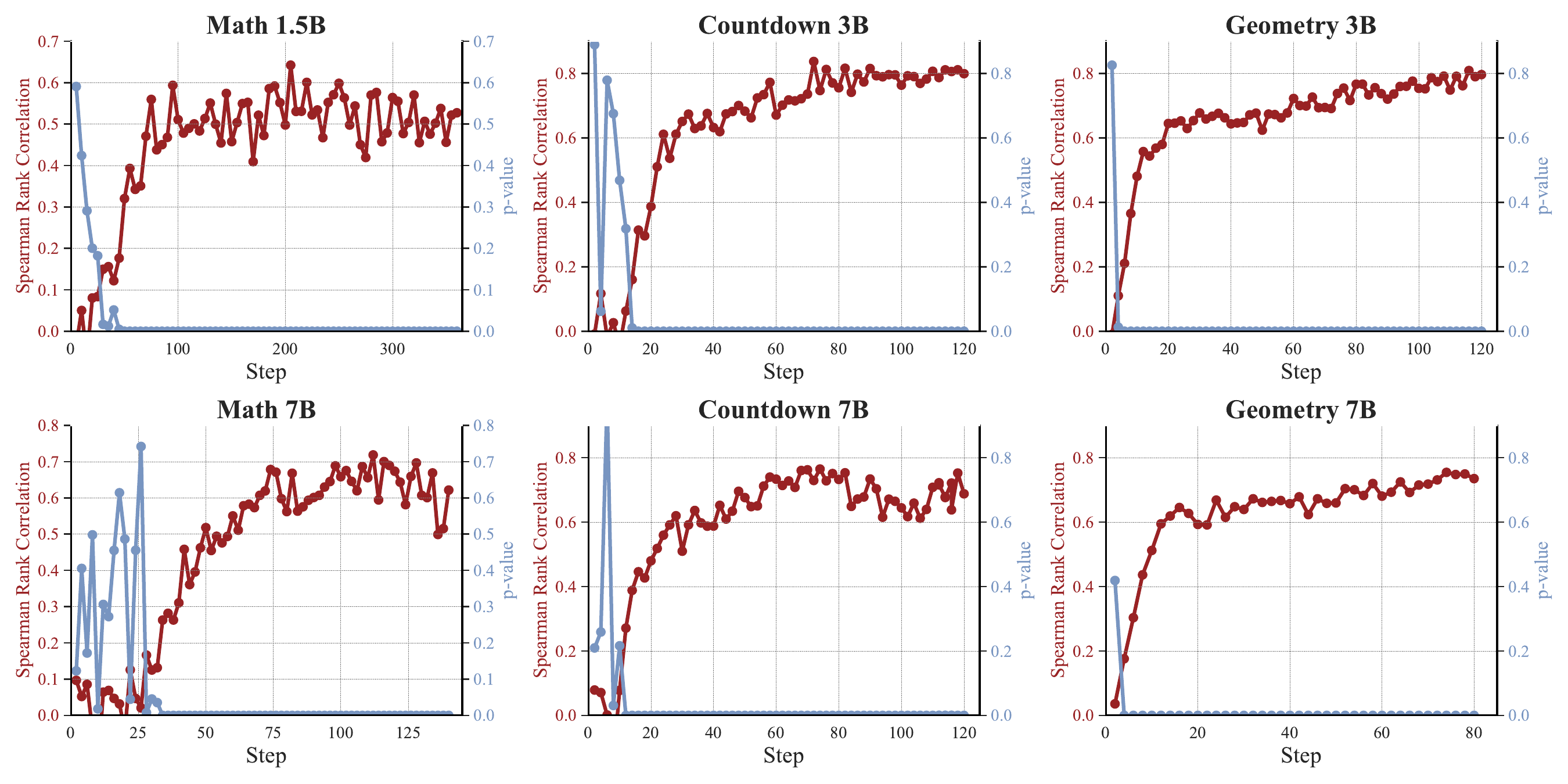}
    \vspace{-15pt}
    \caption{Spearman rank correlation and $p$-value over training steps between the predicted prompt difficulty from our Bayesian surrogate and the empirical success rate. The strong correlation indicates that our method effectively predicts prompt difficulty without incurring costly LLM inferences.
}
    \Description{Spearman rank correlation and $p$-value over training steps between the predicted prompt difficulty from our Bayesian surrogate and the empirical success rate. The strong correlation indicates that our method effectively predicts prompt difficulty without incurring costly LLM inferences.
}
    \vspace{-10pt}
    \label{fig:corr}
\end{figure*}

A central insight of this work is that the difficulty of prompts, quantified as the success rate under the current policy, can be dynamically predicted without additional LLM inference.
To rigorously assess the prediction's fidelity, we adopt Spearman's rank correlation coefficient~\citep{sedgwick2014spearman} $\rho$ as the metric, which quantifies the strength and direction of the monotonic relationship between two sequences by computing the Pearson correlation~\citep{cohen2009pearson} on their ranks:
\begin{equation}
\rho = \frac{\mathrm{cov}(\mathrm{rank}(\hat{\Gamma}^\mathcal{B}), \mathrm{rank}(\widetilde\Gamma^\mathcal{B}))}{\sigma_{\mathrm{rank}(\hat{\Gamma}^\mathcal{B})} \cdot \sigma_{\mathrm{rank}(\widetilde\Gamma^\mathcal{B})}},
\end{equation}
where $\hat{\Gamma}^\mathcal{B} = {\hat\gamma_\tau}^\mathcal{B}$ and $\widetilde\Gamma^\mathcal{B} = {\widetilde\gamma_\tau}^\mathcal{B}$ respectively denote the predicted and empirically estimated success rates, and $\mathrm{rank}(\cdot)$ returns the rank ordering of elements.
To assess statistical significance, we report the $p$-value under the null hypothesis testing that $\hat\gamma$ and $\widetilde\gamma$ are independent; lower values indicate stronger correlation.

In Fig.~\ref{fig:corr}, MoPPS exhibits consistently high rank correlation ($\rho > 0.5$) between the estimated difficulty and the ground-truth with extremely low $p$-values across reasoning tasks and diverse backbone models.
Besides, a clear training progresses can be observed that the correlation steadily improves until stabilizing at a high level.
This validates that our Bayesian surrogate accumulates meaningful evidence over time, progressively refining its belief about prompt difficulty.

\begin{tcolorbox}[takeaway]
MoPPS effectively models and tracks prompt difficulty over time. This capability underpins the efficiency and stability of our posterior-guided sampling strategy, supporting dynamically informative prompt selection without requiring additional LLM inference.
\end{tcolorbox}

\begin{figure*}[t]
    \centering
    \includegraphics[width=0.75\linewidth]{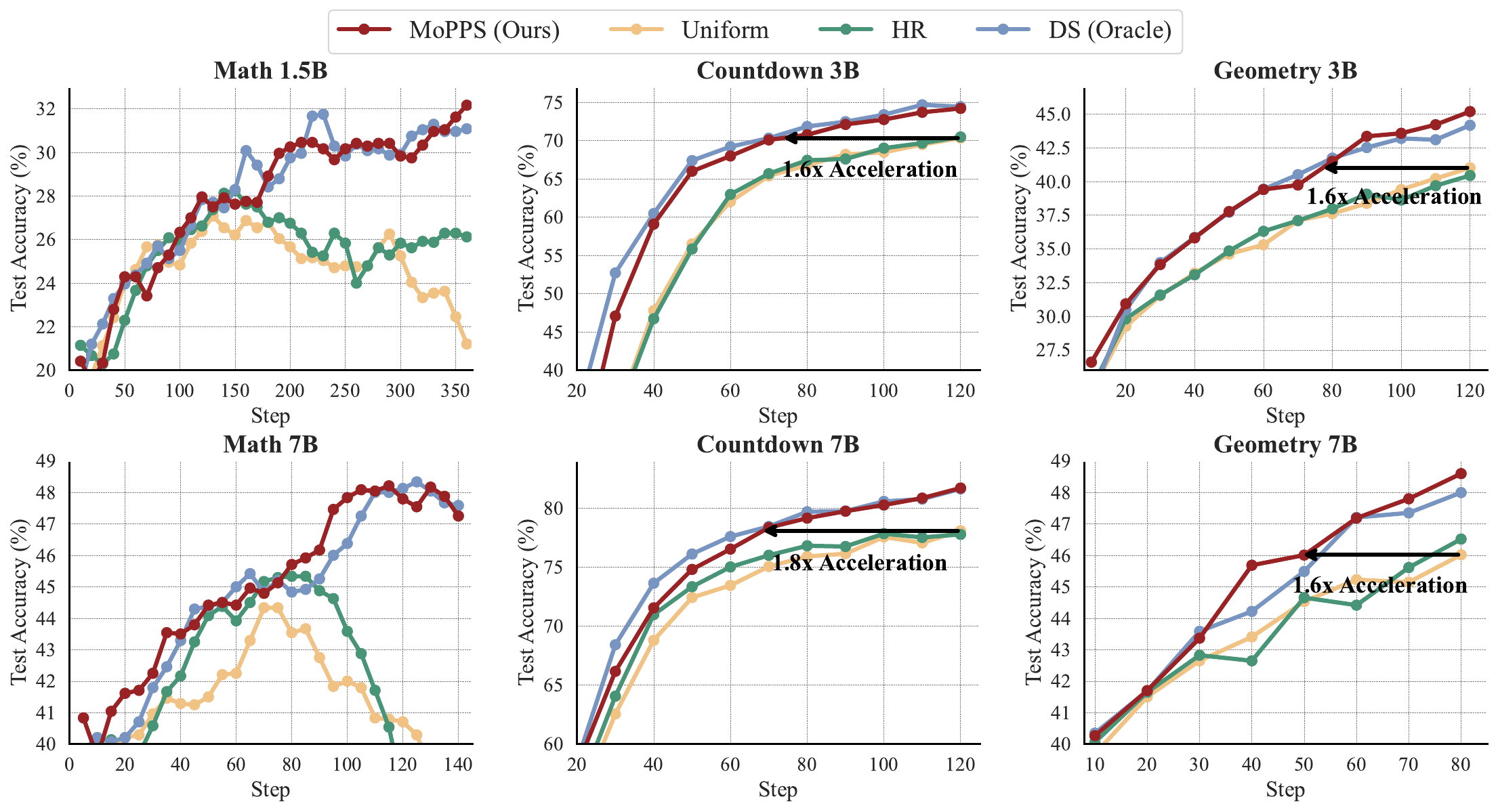}
    \vspace{-15pt}
    \caption{
    Training curves of MoPPS and baselines across three reasoning tasks with varying backbone sizes. 
    Notably, DS serves as an oracle baseline, as it relies on expensive exact LLM evaluations and demands significantly more rollouts.}
    \label{fig:performance}
\end{figure*}

\begin{table*}[t]
\renewcommand{\arraystretch}{1.2}
  \centering
  \caption{Evaluation across mathematics benchmarks. ‘+’ indicates finetuning with the corresponding method. Accuracy is computed as the average \texttt{pass@1} over 16 independent generations per problem. ‘Avg.’ denotes average accuracy across benchmarks, and ‘Rollouts’ indicates the number of rollout samples during finetuning. \textbf{Bold} indicates the best result; \underline{underlined} indicates the second best.
  }
  \vspace{-10pt}
  \resizebox{0.7\linewidth}{!}{
    \begin{tabular}{rccccccc}
    \toprule
    \multicolumn{1}{c}{\textbf{Method}} & \textbf{AIME24} & \textbf{AMC23} & \textbf{MATH500} & \textbf{Minerva.} & \textbf{Olympiad.} & \textbf{Avg. $\uparrow$} &\textbf{Rollouts $\downarrow$}\\
    \midrule
    \textbf{R1-Distill-Qwen-1.5B} & 18.33  & 51.73  & 76.64  & 23.83  & 35.31  & 41.17 & - \\
    \textbf{+Uniform} & 26.46  & 63.18  & 82.78  & 27.46  & 43.00  & 48.57  & \textbf{737k}\\
   \textbf{+HR} & 28.13  & 64.61  & 82.88  & 27.37  & 43.15  & 49.23  & \textbf{737k}\\
   \textbf{+DS (Oracle)} & \underline{31.88}	&\textbf{67.32}	&\underline{84.79}&	\textbf{29.18}	&\textbf{46.83}	&\textbf{52.00} & 2933k\\
\rowcolor{myhighlight}
    \textbf{+MoPPS (Ours)}  & \textbf{32.92}  & \underline{66.72}  & \textbf{84.82}  & \underline{28.81}  & \underline{45.89}  & \underline{51.83} & \textbf{737k} \\
    \midrule
    \textbf{R1-Distill-Qwen-7B} & 37.71  & 68.45  & 86.94  & 34.74  & 46.94  & 54.95 & -  \\
    \textbf{+Uniform} & 45.83  & 73.57  & 89.06  & 37.68  & 50.42  & 59.31  & \textbf{287k} \\
    \textbf{+HR} & 46.46  & 75.98  & 90.01  & \underline{37.94}  & 51.50  & 60.38  & \textbf{287k}\\
    \textbf{+DS (Oracle)} & \textbf{49.79} & \underline{78.99}  & \underline{90.96}  & 37.89  & \textbf{54.45} & \textbf{62.42} & 1147k \\
\rowcolor{myhighlight}
    \textbf{+MoPPS (Ours)} & \underline{48.54}  & \textbf{79.22} & \textbf{91.04} & \textbf{38.49} & \underline{53.69}  & \underline{62.20}  & \textbf{287k} \\
    \bottomrule
    \end{tabular}%
    }
  \label{tab:matheval}%
\end{table*}%

\vspace{-1pt}
\subsubsection{Accelerated RL Finetuning}

We compare the training performance of MoPPS with baselines across different scenarios and backbone models.
Fig.~\ref{fig:performance} shows the training curves, and Table~\ref{tab:matheval} summarizes the final evaluation results on the mathematics task.
Thanks to reliable difficulty prediction, the proposed MoPPS method achieves both training acceleration and better final performance compared to uniform prompt selection.
On MATH, uniform selection suffers from performance collapse, likely due to entropy collapse~\citep{liu2025prorl}.
Effective online prompt selection methods can mitigate this issue and sustain continuous progress, with MoPPS achieving approximately $\frac{32.92 - 26.46}{26.46} \approx \bm{24.4\%}$ relative improvement on AIME24 and about $\bm{6.7\%}$ average relative improvement on multiple benchmarks with the 1.5B backbone compared to Uniform.
On Countdown and Geometry, MoPPS consistently accelerates training across backbone sizes, reaching nearly $\bm{1.8 \times}$ speedup.
Compared to DS, MoPPS attains comparable performance with only  $\bm{25\%}$ rollouts on MATH (Table~\ref{tab:matheval}) and $\bm{21\%}$ on Countdown (Table~\ref{tab:cdeval}).
This efficiency gain stems from DS’s requirement to evaluate a larger candidate prompt set using costly LLM inference at each step, whereas MoPPS amortizes this via lightweight model prediction.

\begin{tcolorbox}[takeaway]
MoPPS significantly outperforms uniform prompt selection, accelerating training and improving overall performance. Meanwhile, it matches DS’s performance with substantially fewer costly LLM rollouts.
\end{tcolorbox}

\subsection{Additional Analysis}

\subsubsection{Algorithm Compatibility}\label{exp:algocompatible}
We assess the compatibility of MoPPS with RL algorithms beyond GRPO by integrating it with two alternative algorithms, \textbf{PPO}~\citep{schulman2017proximal} and \textbf{Reinforce++}~\citep{hu2025reinforce++}, on the Countdown task.
As shown in Table~\ref{tab:algo} and Fig.~\ref{fig:algocurve}, MoPPS consistently improves both training efficiency and final performance compared to uniform selection, regardless of the underlying RL algorithm or whether group generation ($k>1$) is used.
These results confirm that MoPPS is algorithm-agnostic and can be seamlessly integrated into diverse RL finetuning pipelines to enhance sample efficiency.

\subsubsection{The Effects of Prior Knowledge and Selection Strategies}\label{sec:selection}
Our default implementation adopts a uniform prior, $\text{Beta}(1,1)$, and a $\text{Top-}\mathcal{B}$ selection strategy.
To examine the flexibility of our method, we evaluate (i) an alternative \textbf{Threshold} selection strategy, as used in \cite{bae2025online}, which samples prompts with predicted success rates falling within a fixed interval, i.e., $\gamma_{min}\le\hat{\gamma}_\tau\le\gamma_{max}$, and (ii) the integration of \textbf{prior knowledge} (`w/ prior') by pre-evaluating all prompts using the base model and initializing the Beta parameters $\{\alpha, \beta\}$ accordingly.
As shown in Fig.~\ref{fig:ablate}(a), both strategies improve over uniform selection, and $\text{Top-}\mathcal{B}$ performs better.
Incorporating prior knowledge further enhances training efficiency, though our method remains effective even without such prior.

\subsubsection{Role of Online Posterior Updates}\label{sec:offline}
To assess the importance of online posterior updates, we consider an \textbf{Offline} variant that uses only prior knowledge for prompt selection without updating the posterior during training.
As shown in Fig.~\ref{fig:ablate}(b), the offline variant benefits from strong initialization and performs competitively in early stages, but its accuracy degrades later.
This degradation arises from the its inability to adapt to the evolving policy, resulting in outdated difficulty estimates and increasingly suboptimal prompt selection. 
As evidenced in Fig.~\ref{fig:ablate}(c), the offline variant suffers a decline in correlation over time, whereas the online variant improves continually by updating its posterior with new feedback.

\begin{figure*}[t]
    \centering
    \begin{minipage}[c]{0.275\linewidth}
        \centering
\captionof{table}{Evaluation on Countdown with PPO and Reinforce++ using Qwen2.5-3B.
MoPPS consistently improves \texttt{pass@1} accuracy on CD-34 and the harder CD-4 benchmark compared to uniform selection.}
\renewcommand{\arraystretch}{1.2}
\resizebox{\linewidth}{!}{
\begin{tabular}{cccc}
\toprule
\textbf{RL Algos} & \textbf{Benchmarks} & \textbf{Uniform} & \textbf{MoPPS} \\
\midrule
\multirow{2}[1]{*}{\textbf{PPO ($k=1$)}} & CD-34 & 62.33& \textbf{69.12} \\
      & CD-4   & 32.65 & \textbf{40.11} \\
\midrule
\multirow{2}[1]{*}{\textbf{PPO ($k=8$)}} & CD-34 & 72.49 & \textbf{75.50} \\
      & CD-4   & 44.48 & \textbf{48.28} \\
\midrule
\multirow{2}[1]{*}{\textbf{Reinforce++}} & CD-34 & 72.02 & \textbf{74.18} \\
      & CD-4   & 43.36 & \textbf{45.61} \\
\bottomrule
\end{tabular}
}
\label{tab:algo}
    \end{minipage}
    \hspace{6pt}
    \begin{minipage}[c]{0.7\linewidth}
        \centering
    \includegraphics[width=\linewidth]{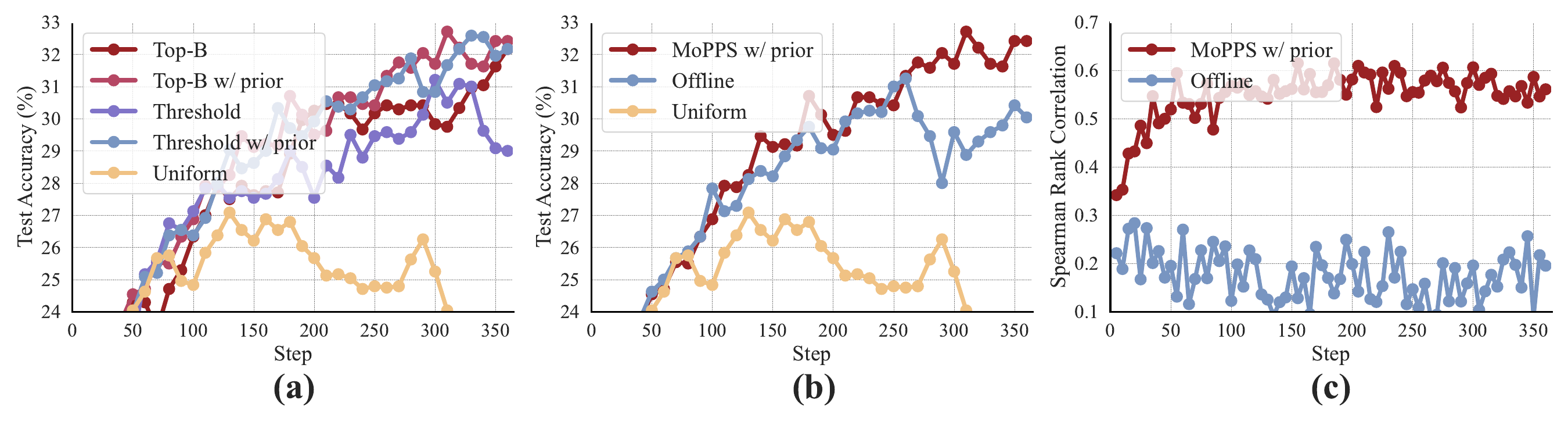}
    \vspace{-28pt}
    \caption{
Ablation studies on selection strategies, prior knowledge, and online posterior updates on MATH using R1-Distill-Qwen-1.5B.
(a) Comparison of $\text{Top-}\mathcal{B}$ and Threshold selection strategies, with and without prior knowledge.  
(b) Training performance of offline selection (prior only) versus MoPPS with prior and online posterior updates.  
(c) Spearman rank correlation over training steps for offline variant and MoPPS with prior.
}
    \label{fig:ablate}
    \end{minipage}%
    
    \vspace{-5pt}
\end{figure*}

\begin{tcolorbox}[takeaway]
MoPPS is compatible with various RL algorithms, consistently improving training efficiency.
Moreover, it is flexible and effective across different selection strategies, benefits further from prior knowledge, and crucially relies on online posterior updates.
\end{tcolorbox}

\section{Related Works}\label{sec:related}

\paragraph{RL Finetuning of LLMs.}
Reinforcement learning has emerged as a powerful paradigm for aligning LLMs with desired behaviors.
Reinforcement Learning with Human Feedback (RLHF), has demonstrated remarkable success in improving instruction-following capabilities and ensuring the safety of LLMs~\citep{dong2024rlhf, dai2023saferlhf, sun2023aligningrlhf, zheng2023secretsrlhf}. 
More recently, advances show that Reinforcement Learning with Verifiable Rewards (RLVR)~\citep{jaech2024openai, guo2025deepseek, team2025kimi, chu2025sftrl, tinyzero} significantly enhances LLMs' reasoning capabilities in structured domains, such as mathematics, where reward signals can be automatically verified.
Among RL algorithms, PPO~\citep{schulman2017proximal} remains a widely adopted method. GRPO~\citep{shao2024deepseekmath} has gained traction more recently as it eliminates the computationally expensive value network in PPO by estimating advantages using a lightweight group-normalized manner.
Several recent works further improve these algorithms by avoiding bias and training collapse, reducing overhead, and enhancing sample efficiency~\citep{yuan2025vcppo, yue2025vapo, liu2025understanding, yu2025dapo, kazemnejad2024vineppo, hu2025reinforce++}.
Moreover, many efforts push the performance frontier across diverse domains and model scales~\citep{luo2025deepscaler, dang2025reinforcement, luo2025deepcoder, zeng2025simplerl, meng2025mm, xu2024llava}, while others provide infrastructure for scalable RL-based LLM training~\citep{sheng2024hybridflow}.

\paragraph{Prompt Selection for RL Finetuning.}
Data curation has emerged as a promising approach to improve training efficiency of RL finetuning.
Offline filtering methods select prompts before training based on criteria like difficulty or diversity~\citep{ye2025limo, li2025limr, wen2025light, hu2025open, yang2024qwen2math, fatemi2025concise, wang2025oneexample}, but often incur extra overhead for prompt assessment and lack adaptivity to evolving training dynamics, as discussed in Sec.~\ref{sec:offline}.
To overcome this, recent studies~\citep{yu2025dapo, zhang2025srpo} have explored online sampling, which selects informative prompts based on the current policy.
Many approaches perform per-step selection by filtering ineffective prompts~\citep{yu2025dapo, liu2025prorl, cui2025process, meng2025mm} or prioritizing moderate difficulty~\citep{bae2025online}, which reduces training steps but requires additional costly LLM evaluations.
Other methods apply per-epoch filtering~\citep{zhang2025srpo, zheng2025act} to avoid per-step evaluation, but weakens adaptability.
SEC~\citep{chen2025self} avoids real evaluations to construct curriculum by estimating category advantages, but it depends on predefined prompt categories.
In contrast, our method enables efficient step-wise prompt selection by amortizing prompt evaluation with posterior-based prediction.
This allows MoPPS to achieve training acceleration, while completely avoiding additional LLM inference cost.

\section{Conclusion}
This work introduces Model Predictive Prompt Selection, a lightweight and effective framework for accelerating RL finetuning of reasoning models through online prompt selection.
By modeling prompt success rates as latent variables and applying recursive Bayesian updates, MoPPS efficiently predicts prompt difficulty without extra LLM inference, enabling reliable and adaptive prompt selection during training.
Experiments on diverse reasoning tasks show that MoPPS consistently improves training efficiency over uniform selection and matches or outperforms evaluation-heavy methods, while substantially reducing LLM rollout costs.

\paragraph{Limitations and Future Work.}
MoPPS adopts a Bernoulli bandit formulation, which assumes approximately binary reward signals.
This work has validated its effectiveness on richer reward types such as format rewards.
Future work will further explore extending MoPPS to handle more complex reward structures like process-based rewards, thereby broadening its applicability to a wider range of scenarios.


\begin{acks}
This work was supported by the National Natural Science Foundation of China (NSFC) with the Number \# 62306326, the National Key R\&D Program of China under Grant 2018AAA0102801, and the Fundamental and Interdisciplinary Disciplines Breakthrough Plan of the Ministry of Education of China (JYB2025XDXM503).
We thank all reviewers for their insightful comments and constructive suggestions.
\end{acks}


\bibliographystyle{ACM-Reference-Format}
\balance
\bibliography{sample-base}

\appendix

\section*{Appendix Overview}

\begin{figure*}[ht]
    \centering
    \includegraphics[width=0.9\linewidth]{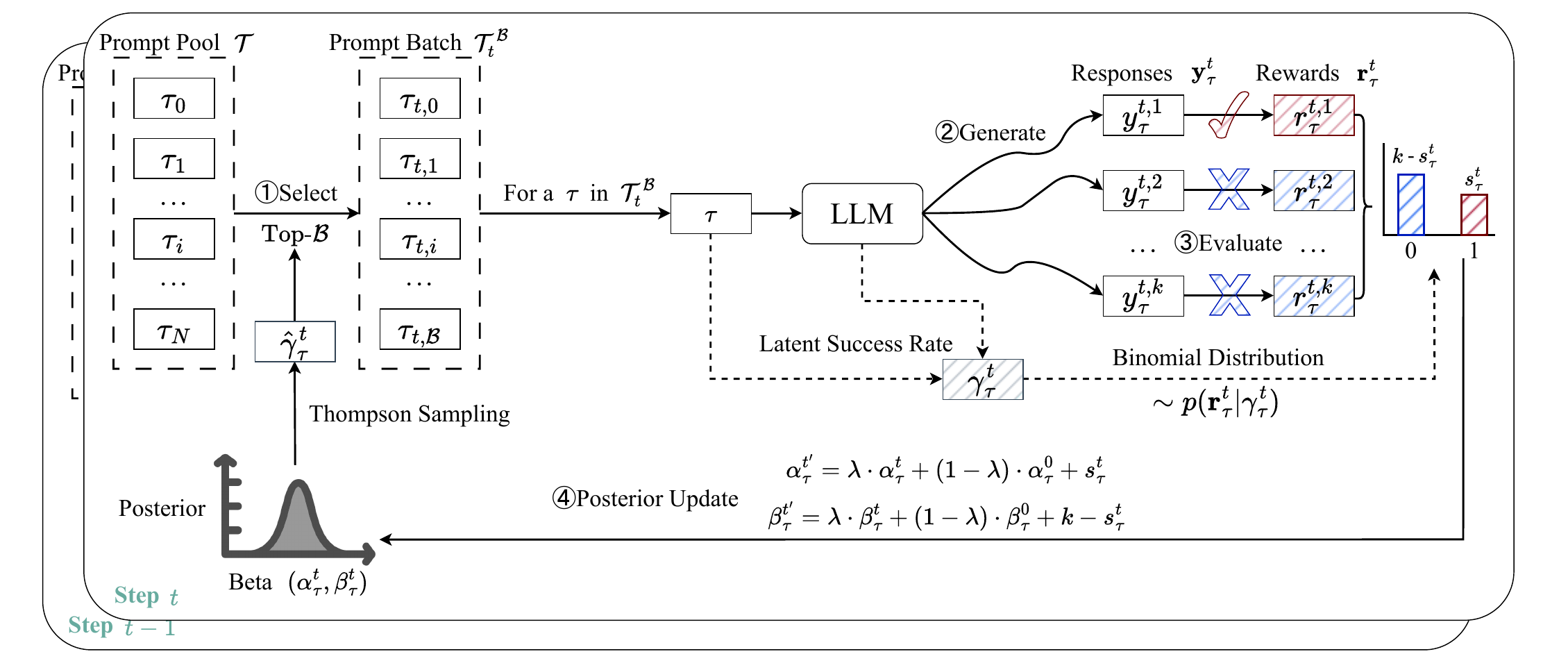}
    \caption{Illustration of the prompt selection, LLM generation, prompt evaluation, and posterior update process during RL finetuning.
At each step, a batch of prompts is actively selected based on the predicted success rates $\hat{\gamma}_\tau^t$ using Thompson Sampling.
Then, the model generates multiple responses per prompt and receives binary rewards drawn from a Binomial distribution parameterized by the latent success rate $\gamma_\tau^t$.
These observations are used to update the Beta posterior in a recursive manner.}

    \label{fig:notation}
\end{figure*}
This appendix provides additional details and analyses which is organized as follows:
\begin{itemize}[leftmargin=10pt]
  \item \textbf{Appendix~\ref{appsec:notations} (Notations Explanation):}  
  illustrates the key notations used throughout the paper by walking through one RL finetuning step of reasoning models.
  \item \textbf{Appendix~\ref{appsec:proof} (Theoretical Proof):}  
  Presents a formal proof and analysis of the success rate estimation bound in MoPPS.
  \item \textbf{Appendix~\ref{appsec:implementation} (Implementation Details):}  
  provides comprehensive information on the experimental setup, including datasets, reward functions, backbones, and training configurations for both baselines and MoPPS.
  \item \textbf{Appendix~\ref{appsec:results} (Additional Results):}  
  reports evaluation results across benchmarks, rollout efficiency analysis, ablation studies (e.g., temporal discounting), and various behavioral analyses (e.g., prompt length, response length, and reduction of ineffective prompts).
  \item \textbf{Appendix~\ref{appsec:dataexample} (Data Examples):}  
  shows representative prompt templates across tasks.
\end{itemize}

\section{Notations Explanation}\label{appsec:notations}
To assist readers unfamiliar with RL finetuning of reasoning models, we provide an overview of a training step in Fig.~\ref{fig:notation} and clarify key notations used throughout the paper.

At each RL training step $t$, a batch of prompts $\mathcal{T}^\mathcal{B}_t$ is firstly sampled from the full prompt pool $\mathcal{T}$ according to the prompt selection method, i.e., Thompson Sampling combined with $\text{Top-}\mathcal{B}$ selection in the proposed MoPPS.
Then, for each prompt $\tau$ in the batch, the current policy (LLM) takes it as input and generates multiple responses $\bm y^t_\tau = \{y_\tau^{t,j}\}_{j=1}^k$ via autoregressive decoding.
Each response $y_\tau^{t,j}$ is evaluated against the ground-truth answer and given a binary reward $r_{\tau}^{t,j}$: $1$ for correct, and $0$ otherwise.
We denote the reward vector for prompt $\tau$ as $\bm r^t_{\tau} = \{r_{\tau}^{t,j}\}_{j=1}^k$.
As we stated in the main text, we associate each prompt $\tau$ with a success rate $\gamma_\tau^t \in [0,1]$ and treat it as the latent variable, which reflects the chance of $\tau$'s problem-solving success under the current policy.
Hence, the likelihood of observing $\bm{r}^t_\tau$ given $\gamma^t_\tau$ follows a binomial form:
\begin{equation}
\begin{aligned}
&p(r_{\tau}^{t,i})=(\gamma^t_\tau)^{[r_{\tau}^{t,i}=1]}\cdot(1-\gamma^t_\tau)^{[r_{\tau}^{t,i}=0]}\Rightarrow\\
&p(\bm r^t_\tau \mid \gamma^t_\tau)
= \binom{k}{s^t_\tau}\cdot(\gamma^t_\tau)^{s^t_\tau}\cdot(1 - \gamma^t_\tau)^{k - s^t_\tau} \
\text{with}
\
s^t_\tau \triangleq \sum_{j=1}^k r_\tau^{t,j}.
\end{aligned}
\label{appeq:likelihood}
\end{equation}

Given these rewards, we adopt a recursive Bayesian mechanism to update the posterior distribution over the success rate $\gamma^t_\tau$, as detailed in Sec.~\ref{sec:bayesian}.

Examples of prompts $\tau$ for different tasks are provided in Appendix~\ref{appsec:dataexample}, and the reward function details are described in Appendix~\ref{appsec:implementation}.
Amortized prompt evaluation in this work refers to the use of a surrogate model to simulate the probabilistic outcome of LLM inference given a specific prompt.

\section{Theoretical Proof}\label{appsec:proof}

\posteriormeanbound*

\begin{proof}

By the update rule in Eq.~(\ref{eq:betaupdate}), at each time a task $\tau$ is selected, the total pseudo count evolves as:
\begin{equation}
     C_\tau^t := \alpha_\tau^{t'}+\beta_\tau^{t'} = \lambda C_\tau^{t-1} + (1-\lambda)(\alpha_\tau^0+\beta_\tau^0)+k ,\quad t\geq0.
\end{equation}
with $C_\tau^{-1}:=\alpha_\tau^0+\beta_\tau^0$.

To analyze the monotonicity of $\{C_\tau^t\}_{t\ge0}$, we consider an auxiliary sequence $\{x_n\}_{n\ge0}$ that follows the same update rule at every step $n$.
Specifically, define the recurrence
\begin{equation}
x_n = \lambda x_{n-1} + (1 - \lambda) x_{-1} + k,
\end{equation}
which has the closed-form:
\begin{equation}
x_n = \frac{1 - \lambda^n}{1 - \lambda}k + x_{-1}.
\end{equation}
where $x_{-1}\in \mathbb{R}$ is a fixed constant.

This shows that $\{x_n\}_{n\ge0}$ is strictly increasing since $k > 0,\ m>0$ and $\lambda \in (0,1)$.

The actual sequence $\{C_\tau^t\}_{t\ge0}$ is only updated at time steps when task $\tau$ is selected, and remains unchanged otherwise. Each time it is updated, it follows the same rule as $\{x_n\}_{n\ge0}$.
Therefore, $\{C_\tau^t\}_{t\ge0}$ is non-decreasing in $t$.

Define the empirical success probability as:
\begin{equation}
    \bar{s}^t_{\tau}=\frac{s^t_{\tau}}{k},
\end{equation}
and define the prior posterior mean estimate as:
\begin{equation}
    \bar{\gamma}^{-1}_\tau := \frac{\alpha^{0}_\tau}{\alpha^{0}_\tau + \beta^{0}_\tau}
\end{equation}

Given the definition of $C_\tau^t$, the posterior mean can be written as a convex combination:
\begin{equation}
\begin{aligned}
    \bar{\gamma}^{t}_\tau&=\frac{\lambda \alpha^{t}_{\tau} + (1 - \lambda) \alpha_{\tau}^{0}  + s^t_\tau}{C_\tau^t}\\
    &=\frac{\lambda C_\tau^{t-1} \cdot \bar{\gamma}^{t-1}_\tau + (1 - \lambda) (\alpha_{\tau}^{0} + \beta_{\tau}^{0})\cdot  \bar{\gamma}_\tau^{-1} + k\cdot \bar{s}^t_{\tau}}{C_\tau^t}.
\end{aligned}
\end{equation}
or equivalently,
\begin{equation}
    \bar{\gamma}^{t}_\tau=w_1\cdot \bar{\gamma}^{t-1}_\tau + w_2\cdot  \bar{\gamma}_\tau^{-1} + w_3 \cdot \bar{s}^t_{\tau}
\end{equation}
where $$w_1=\frac{\lambda C_\tau^{t-1}}{C_\tau^t},\  w_2=\frac{(1 - \lambda) (\alpha_{\tau}^{0} + \beta_{\tau}^{0})}{C_\tau^t},\  w_3=\frac{k}{C_\tau^t}, \ w_1+w_2+w_3=1$$

We now bound the estimation error:
\begin{equation}
    |\bar{\gamma}^{t}_\tau-\gamma^{t}_\tau|\le w_1\cdot |\bar{\gamma}^{t-1}_\tau-\gamma^{t}_\tau|+w_2\cdot |\bar{\gamma}_\tau^{-1}-\gamma^{t}_\tau|+w_3\cdot |\bar{s}^t_{\tau}-\gamma^{t}_\tau|.
\end{equation}

Apply the triangle inequality to the first term:
\begin{equation}
    |\bar{\gamma}^{t-1}_\tau-\gamma^{t}_\tau|\le |\bar{\gamma}^{t-1}_\tau-\gamma^{t-1}_\tau| + |\gamma^{t-1}_\tau-\gamma^{t}_\tau| \le \epsilon_{t-1} + \delta.
\end{equation}

Since $\bar{\gamma}_\tau^{-1}, \gamma^t_\tau \in [0,1]$, 
the second term satisfies:
\begin{equation}
    |\bar{\gamma}_\tau^{-1}-\gamma^{t}_\tau|\le 1.
\end{equation}

Finally, by Hoeffding’s inequality~\citep{hoeffding1994probability} for the binomial mean $\bar{s}^t_\tau \sim \text{Binomial}(k, \gamma^t_\tau)/k$:
\begin{equation}
\mathbb{P}\left(|\bar{s}^t_\tau - \gamma^t_\tau| \ge \eta\right) \le 2 \exp(-2 k \eta^2).
\end{equation}

Combining the above, with probability at least $1-2 \exp(-2k\eta^2)$:
\begin{equation}
    \epsilon_t = |\bar{\gamma}^{t}_\tau-\gamma^{t}_\tau|\le w_1\cdot(\epsilon_{t-1} + \delta) + w_2\cdot 1+ w_3\cdot \eta
\end{equation}

Since $\{C_\tau^t\}_{t\ge0}$ is non-decreasing in $t$, it follows that
\begin{equation}
    w_1=\lambda\cdot\frac{C_\tau^{t-1}}{C_\tau^t}\le\lambda.
\end{equation}

Assume $k\ge (\alpha_{\tau}^{0} + \beta_{\tau}^{0})$, which is reasonable since usually $(\alpha_{\tau}^{0}, \beta_{\tau}^{0})=(1,1)$. Thus,  we have $C_\tau^1=k+C_\tau^0\ge 2C_\tau^0$. So: 
\begin{equation}
    w_2=(1 - \lambda) \cdot \frac{(\alpha_{\tau}^{0} + \beta_{\tau}^{0})}{C_\tau^t}\le(1 - \lambda) \cdot \frac{C_\tau^0}{C_\tau^1}\le \frac{1-\lambda}{2}.
\end{equation}

Therefore, since 
$w_3=\frac{k}{C_\tau^t}<1$
,
with probability at least $1-2 \exp(-2k\eta^2)$, the error bound can be derived as:
\begin{equation}
    |\bar{\gamma}^{t}_\tau-\gamma^{t}_\tau|< \lambda\cdot (\epsilon_{t-1} + \delta) + \frac{(1-\lambda)}{2} +\eta.
\end{equation}

This completes the proof.
\end{proof}

\paragraph{Implication under Near-Stationarity.}
In nearly stationary regimes where model updates are small and $\delta \approx 0$, setting $\lambda \to 1$ simplifies the error bound to
\begin{equation}
|\bar{\gamma}^{t}_\tau - \gamma^{t}_\tau| \lessapprox \epsilon_{t-1}+\eta.
\end{equation}
This suggests that the estimation error is bounded by the previous-step error $\epsilon_{t-1}$ and the sampling tolerance $\eta$.
Crucially, $\eta$ can be made arbitrarily small by increasing the number of response samples, thereby controlling the approximation error.
As a result, when $\eta$ remains small, the overall estimation error $\epsilon_t$ approximately contracts over time, indicating increasingly accurate posterior inference as training stabilizes.

\section{Implementation Details}
\label{appsec:implementation}

\subsection{Tasks}

\subsubsection{Mathematics}

\paragraph{Training Dataset.}
We train models on the training split of the MATH dataset~\citep{hendrycksmath2021}, which contains 7,500 competition-level math problems.
Following \citet{sheng2024hybridflow}, we use the dataset hosted at \url{https://huggingface.co/datasets/DigitalLearningGmbH/MATH-lighteval}.

\paragraph{Evaluation Benchmarks.}
We evaluate on a suite of math benchmarks: AIME24, AMC23, MATH500~\citep{lightman2023let}, Minerva Math~\citep{lewkowycz2022solving}, and OlympiadBench~\citep{he2024olympiadbench}, using the datasets provided by DeepScaler~\citep{luo2025deepscaler}.
Training curves are plotted using performance on AIME24.

\paragraph{Reward Function.}
Following the default setup in verl~\citep{sheng2024hybridflow}, we use a binary reward function that assigns a reward of $1$ for a correct answer and $0$ otherwise.

\subsubsection{Planning}

\paragraph{Training Dataset.}
We adopt the Countdown Number Game, which requires combining given numbers using basic arithmetic operations to reach a target value.
Specifically, we adopt a 2,000-problem subset of the Countdown-34 dataset from \url{https://huggingface.co/datasets/Jiayi-Pan/Countdown-Tasks-3to4} for training.
In Countdown-34, each problem provides either 3 or 4 source numbers.

\paragraph{Evaluation Benchmarks.}
Evaluation is conducted on two benchmarks: a 512-problem held-out split from Countdown-34 (CD-34), and a 512-problem subset from Countdown-4 (CD-4), a more challenging variant from \url{huggingface.co/datasets/Jiayi-Pan/Countdown-Tasks-4}. 
Unlike CD-34, CD-4 provides 4 source numbers per problem, which significantly increases the search space and problem difficulty.
Training curves are plotted using CD-34.

\paragraph{Reward Function.}
Following \cite{tinyzero}, we include a format term in the reward function:
\begin{equation}
r = 
\begin{cases}
1 & \text{if response is correct}, \\
0.1 & \text{if response is incorrect but with correct formatting}, \\
0 & \text{otherwise}.
\end{cases}
\end{equation}

\subsubsection{Geometry}

\paragraph{Training Dataset.}
We train on the 2,101-problem training split of the Geometry3k dataset~\citep{lu2021inter, geometry3k_dataset}, available at \url{https://huggingface.co/datasets/hiyouga/geometry3k}.
Each problem in Geometry3k consists of a geometric diagram and an accompanying natural language question, often requiring multi-step spatial or logical reasoning.

\paragraph{Evaluation Benchmarks.}
Evaluation is conducted on the official 601-problem test split.

\paragraph{Reward Function.}
Following verl~\citep{sheng2024hybridflow}, we use the same reward function as in Countdown.

Appendix~\ref{appsec:dataexample} presents data examples from each of the training datasets.

\subsection{Models}
We adopt six models spanning diverse types and sizes.
All models are obtained from their official Hugging Face repositories and used as released:
\begin{itemize}[leftmargin=10pt]
    \item DeepSeek-R1-Distill-Qwen-1.5B: \url{https://huggingface.co/deepseek-ai/DeepSeek-R1-Distill-Qwen-1.5B};
    \item DeepSeek-R1-Distill-Qwen-7B: \url{https://huggingface.co/deepseek-ai/DeepSeek-R1-Distill-Qwen-7B};
    \item Qwen2.5-3B: \url{https://huggingface.co/Qwen/Qwen2.5-3B};
    \item Qwen2.5-7B: \url{https://huggingface.co/Qwen/Qwen2.5-7B};
    \item Qwen2.5-VL-3B-Instruct: \url{https://huggingface.co/Qwen/Qwen2.5-VL-3B-Instruct};
    \item Qwen2.5-VL-7B-Instruct: \url{https://huggingface.co/Qwen/Qwen2.5-VL-7B-Instruct}.
\end{itemize}

\subsection{Training Details}

We adopt the widely used GRPO~\citep{shao2024deepseekmath}, implemented in the verl framework~\citep{sheng2024hybridflow}, as our default RL algorithm.

At each training step, $k=8$ responses per prompt are sampled to estimate advantages, using temperature 1.0 and $\texttt{top\_p} = 1.0$.
Evaluation is based on \texttt{pass@1}, computed from 16 independent generations per prompt with temperature 0.6 and $\texttt{top\_p} = 0.95$, following \citet{luo2025deepscaler}.
We disable the KL penalty ($\beta=0$) for MATH and Countdown, following \citet{yu2025dapo}, but retain it for Geometry3k with $\beta=0.1$ (3B) and $\beta=0.3$ (7B) to ensure training stability.
Training batch sizes $\mathcal{B}$ are set to 256 for MATH and Countdown with mini-batch sizes of 128 and 64, repectively, and 512 for Geometry3k with a mini-batch size of 256.
The maximum response length is 8192 for MATH, and 1024 for Countdown and Geometry3k.
Entropy regularization is applied with coefficient 0.001, following~\citet{luo2025deepscaler}.
Optimization is performed using Adam~\citep{kingma2014adam} with a learning rate of $1\mathrm{e}{-6}$, beta $(0.9,0.999)$, no warm-up, and weight decay $0.01$.
The Clip-Higher strategy from \citet{yu2025dapo} is applied, which decouples clipping ranges, with $\epsilon_{\text{low}} = 0.2$ and $\epsilon_{\text{high}} = 0.28$.
All experiments are conducted on 8 NVIDIA A100 or H100 80GB GPUs.

\paragraph{Oracle baseline: Dynamic Sampling}
We use the verl implementation, which repeatedly samples candidate prompts, queries LLM rollouts, and filters out prompts with zero reward standard deviation, until $\mathcal{B}$ prompts are selected.

\paragraph{MoPPS}
We set the Beta prior as $(\alpha^0, \beta^0) = (1, 1)$, target success probability $\gamma^* = 0.5$ for all training tasks.
The decay factor $\lambda$ is set to $0.5$ for Countdown and Geometry3k, $1$ for MATH.
Since the primary objective of this work is to address the two research questions, and the current performance is satisfactory, we do not emphasize performance optimization.
We leave extensive hyperparameter tuning for future work.
The candidate batch size $\hat{\mathcal{B}}$ is set to $16\times \mathcal{B}$, which is relatively large.
For Countdown and Geometry3k, this means the candidate prompt batch $\mathcal{T}^{\hat{\mathcal{B}}}$ effectively covers the entire pool $\mathcal{T}$.

\paragraph{MoPPS variants}
All variants share the same base setups unless otherwise specified:
\begin{itemize}[leftmargin=10pt]
    \item Threshold: samples prompts with predicted success rates falling within a fixed interval, i.e., $\gamma_{min}\le\hat{\gamma}_\tau\le\gamma_{max}$, with $\gamma_{min}=0.3,\ \gamma_{max}=0.7$ which is the best setup in \citet{bae2025online}.
    \item MoPPS w/ prior: incorporates prior knowledge by pre-evaluating all prompts using the base model. For each prompt, 8 responses are generated, and the Beta parameters ${\alpha^0, \beta^0}$ are initialized as $\alpha^0 = 1 + 3 \times$ (number of correct responses), $\beta^0 = 1 + 3 \times$ (number of incorrect responses).
    \item Offline: uses only prior knowledge for prompt selection without updating the posterior during training, i.e. $\alpha_\tau^{t+1}=\alpha_\tau^t=\alpha_\tau^0,\ \beta_\tau^{t+1}=\beta_\tau^t=\beta_\tau^0$;
    \item MoPPS with PPO ($k=1$): only one response is generated per task, leading to sparse feedback. To maintain consistent hyperparameter settings and amplify signal strength, the posterior update scales $s_\tau^t$ and $k-s_\tau^t$ by 8, i.e., the default $k$ value in GRPO, in Eq.~\ref{eq:betaupdate}. This adjustment is not required for PPO ($k=8$) or Reinforce++, as both employ group generation.
\end{itemize}

\begin{table}[t]
  \centering
  \caption{Evaluation results on Countdown. Models trained with different prompt selection methods are evaluated on two benchmarks: Countdown-34 (CD-34) and Countdown-4 (CD-4). MoPPS consistently outperforms Uniform without requiring any additional LLM inference, and matches or surpasses DS while using substantially fewer rollouts.}
  \vspace{-5pt}
  \resizebox{\linewidth}{!}{
    \begin{tabular}{ccccccc}
    \toprule
    \multirow{2}[2]{*}{\textbf{Method}} & \multicolumn{3}{c}{\textbf{Qwen2.5-3B}} & \multicolumn{3}{c}{\textbf{Qwen2.5-7B}} \\
          & \textbf{CD-34} & \textbf{CD-4} & \textbf{Rollouts} & \textbf{CD-34} & \textbf{CD-4} & \textbf{Rollouts} \\
    \midrule
    \textbf{Uniform} & 69.87  & 39.42  & \textbf{246k} & 77.84  & 53.27  & \textbf{246k} \\
    \textbf{HR} & 70.19  & 42.10& \textbf{246k} &78.15  & 54.54 & \textbf{246k}\\
    \textbf{DS (Oracle)} & \textbf{74.95} & \textbf{47.67} & 1141k & \underline{81.26}  & \textbf{60.77} & 1006k \\
\rowcolor{myhighlight}
    \textbf{MoPPS (Ours)} & \underline{74.46} & \underline{47.02}  & \textbf{246k} & \textbf{82.18} & \underline{59.16}  & \textbf{246k} \\
    \bottomrule
    \end{tabular}
    }
    \vspace{-10pt}
  \label{tab:cdeval}
\end{table}

\begin{table}[t]
  \centering
  \caption{Evaluation results on Geometry. Models trained with different prompt selection methods are evaluated on the test split of Geometry3k (Geo3k test). MoPPS consistently outperforms Uniform without requiring any additional LLM inference, and surpasses DS while using substantially fewer LLM-generated rollouts.}
  \vspace{-10pt}
  \resizebox{\linewidth}{!}{
    \begin{tabular}{ccccc}
    \toprule
    \multirow{2}[2]{*}{\textbf{Method}} & \multicolumn{2}{c}{\textbf{Qwen2.5-VL-3B-Instruct}} & \multicolumn{2}{c}{\textbf{Qwen2.5-VL-7B-Instruct}} \\
          & \textbf{Geo3k test} & \textbf{Rollouts} & \textbf{Geo3k test} & \textbf{Rollouts} \\
    \midrule
    \textbf{Uniform} & 40.69 & \textbf{492k} & 46.22 & \textbf{328k} \\
    \textbf{HR }& 40.44 & \textbf{492k} & 46.52 & \textbf{328k} \\ 
    \textbf{DS (Oracle)} & \underline{44.33} & 1262k & \underline{48.11} & 782k \\
\rowcolor{myhighlight}
    \textbf{MoPPS (Ours)} & \textbf{45.15} & \textbf{492k} & \textbf{48.24} & \textbf{328k} \\
    \bottomrule
    \end{tabular}
    }
  \vspace{-10pt}
  \label{tab:geoeval}
\end{table}

\begin{table*}[h]
\renewcommand{\arraystretch}{1.2}
  \centering
  \vspace{-10pt}
  \caption{Evaluation across mathematics benchmarks with maximum response length 32,768. ‘+’ indicates finetuning with the corresponding method. Accuracy is computed as the average \texttt{pass@1} over 16 independent generations per problem. ‘Avg.’ denotes average accuracy across benchmarks, and ‘Rollouts’ indicates the number of rollout samples during finetuning. \textbf{Bold} indicates the best result; \underline{underlined} indicates the second best. ‘MoPPS w/ prior’ means incorporating prior knowledge.}
  \resizebox{0.7\linewidth}{!}{
    \begin{tabular}{rccccccc}
    \toprule
    \multicolumn{1}{c}{\textbf{Method}} & \textbf{AIME24} & \textbf{AMC23} & \textbf{MATH500} & \textbf{Minerva.} & \textbf{Olympiad.} &  \textbf{Avg. $\uparrow$} &\textbf{Rollouts $\downarrow$} \\
    \midrule
    \textbf{R1-Distill-Qwen-1.5B} & 28.12 & 61.67 & 83.18 & 26.54 & 43.33 & 48.57&- \\
    \textbf{+Uniform} & 31.46 & 67.70  & 84.22 & 27.94 & 45.06 & 51.28& \textbf{737k}\\
    \textbf{+HR} & 30.42 & 66.49  & 84.30 & 27.53 & 45.06 & 50.76& \textbf{737k}\\
    \textbf{+DS (Oracle)}   & \underline{32.92} & \underline{69.95} & \textbf{86.44} & \textbf{30.26} & \textbf{49.66} & \textbf{53.85}&2933k \\
\rowcolor{myhighlight}
\textbf{+MoPPS (Ours)} & \textbf{35.42} & \textbf{70.33} & \underline{85.25} & \underline{29.73} & \underline{46.68} & \underline{53.48}& \textbf{737k} \\
\rowcolor{gray!10}
    \textbf{+MoPPS w/ prior (Ours)} & 36.67 & 69.58 & 85.71  & 29.57 & 47.76 & 53.86& 797k \\

    \bottomrule
    \end{tabular}
    }
  \label{tab:math32k}
\end{table*}

\section{Additional Results}\label{appsec:results}

\subsection{Evaluation Results}

We evaluate the trained checkpoints on corresponding benchmarks to assess the final performance of different prompt selection methods.
Table~\ref{tab:matheval}, Table~\ref{tab:cdeval}, and Table~\ref{tab:geoeval} present results on mathematics, planning, and geometry tasks, respectively.
The results show that both MoPPS and DS consistently outperform Uniform sampling in terms of final accuracy and training efficiency.
Notably, MoPPS matches or surpasses DS while requiring significantly fewer rollouts, as shown in Fig.~\ref{fig:rolloutcurve}, leading to substantial reductions in LLM inference cost.

We also conduct an out-of-distribution evaluation by testing MATH-trained checkpoints, which are trained with a max response length of 8k, using a much larger response length of 32k.
As shown in Table~\ref{tab:math32k}, MoPPS continues to outperform Uniform and matches Dynamic Sampling in this setting, benefiting substantially from the increased response length.
This indicates that MoPPS can generalize to large-scale training scenarios and achieve better performance.

\subsection{Other Analysis}
\begin{figure}[h]
    \centering
    \includegraphics[width=\linewidth]{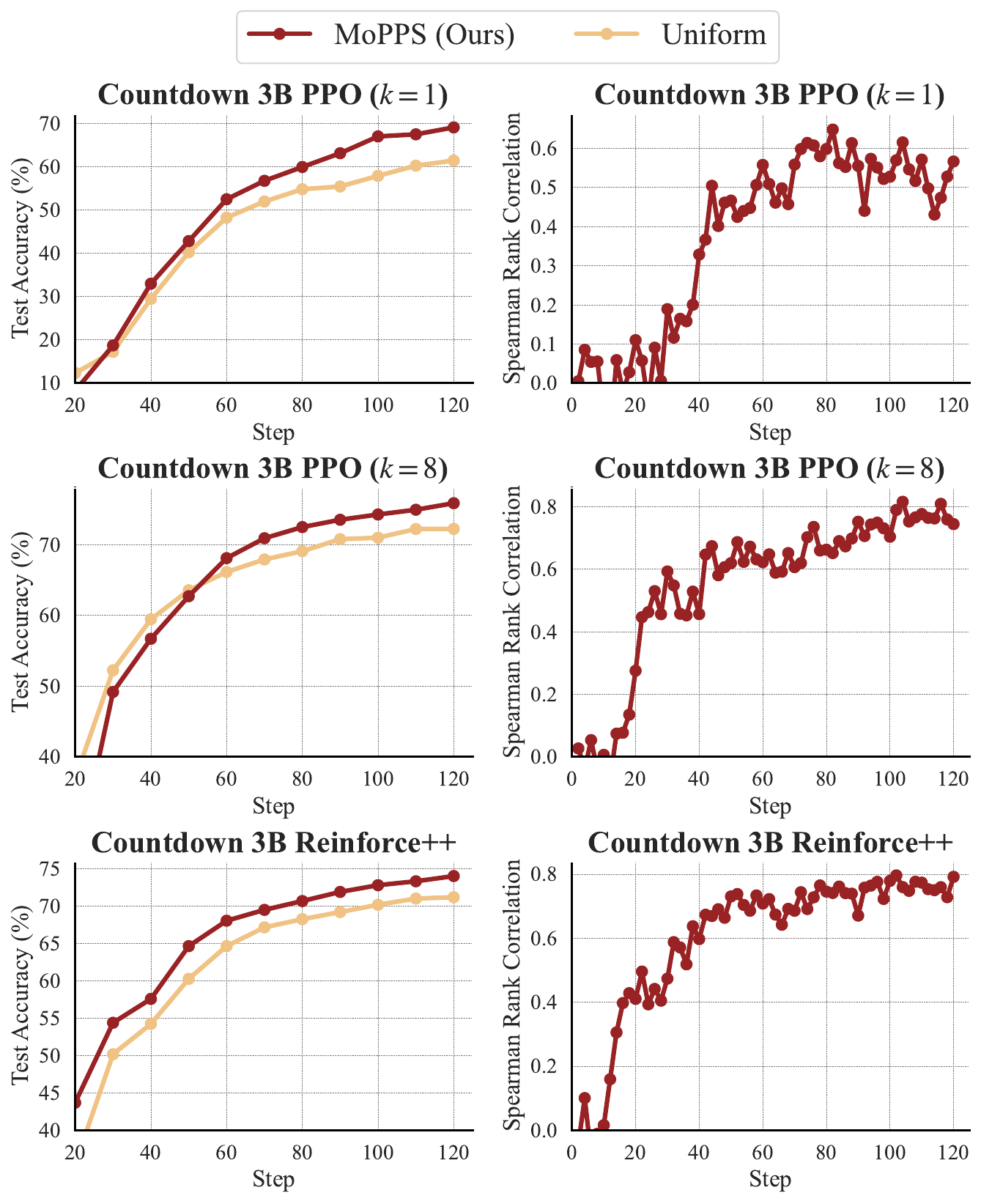}
    \vspace{-25pt}
    \caption{Training curves on Countdown with PPO (both $k=1$ and $k=8$) and Reinforce++. 
MoPPS  consistently accelerates convergence and improves final accuracy compared to uniform prompt selection across different RL algorithms and sampling number settings.
The Spearman rank correlation further shows that MoPPS reliably predicts prompt difficulty across varied setups.}
\vspace{-5pt}
    \label{fig:algocurve}
\end{figure}

\subsubsection{Algorithm Compatibility.}
MoPPS is compatible with various RL algorithms beyond GRPO.
Fig.~\ref{fig:algocurve} presents the training curves of MoPPS integrated with two alternative algorithms, \textbf{PPO}\citep{schulman2017proximal} and \textbf{Reinforce++}~\citep{hu2025reinforce++}, on the Countdown task.
For PPO, we test both the standard single-sample setup ($k=1$) and the group sampling variant ($k=8$), and MoPPS consistently improves convergence speed and final \texttt{pass@1} accuracy under both settings.
The improvements are also evident when combined with Reinforce++, further confirming that MoPPS is algorithm-agnostic and broadly applicable across different RL finetuning pipelines.
Moreover, Spearman rank correlation analysis demonstrates that MoPPS reliably predicts prompt difficulty across these varied algorithmic and sampling configurations.

\begin{figure}[t]
    \centering
    \includegraphics[width=\linewidth]{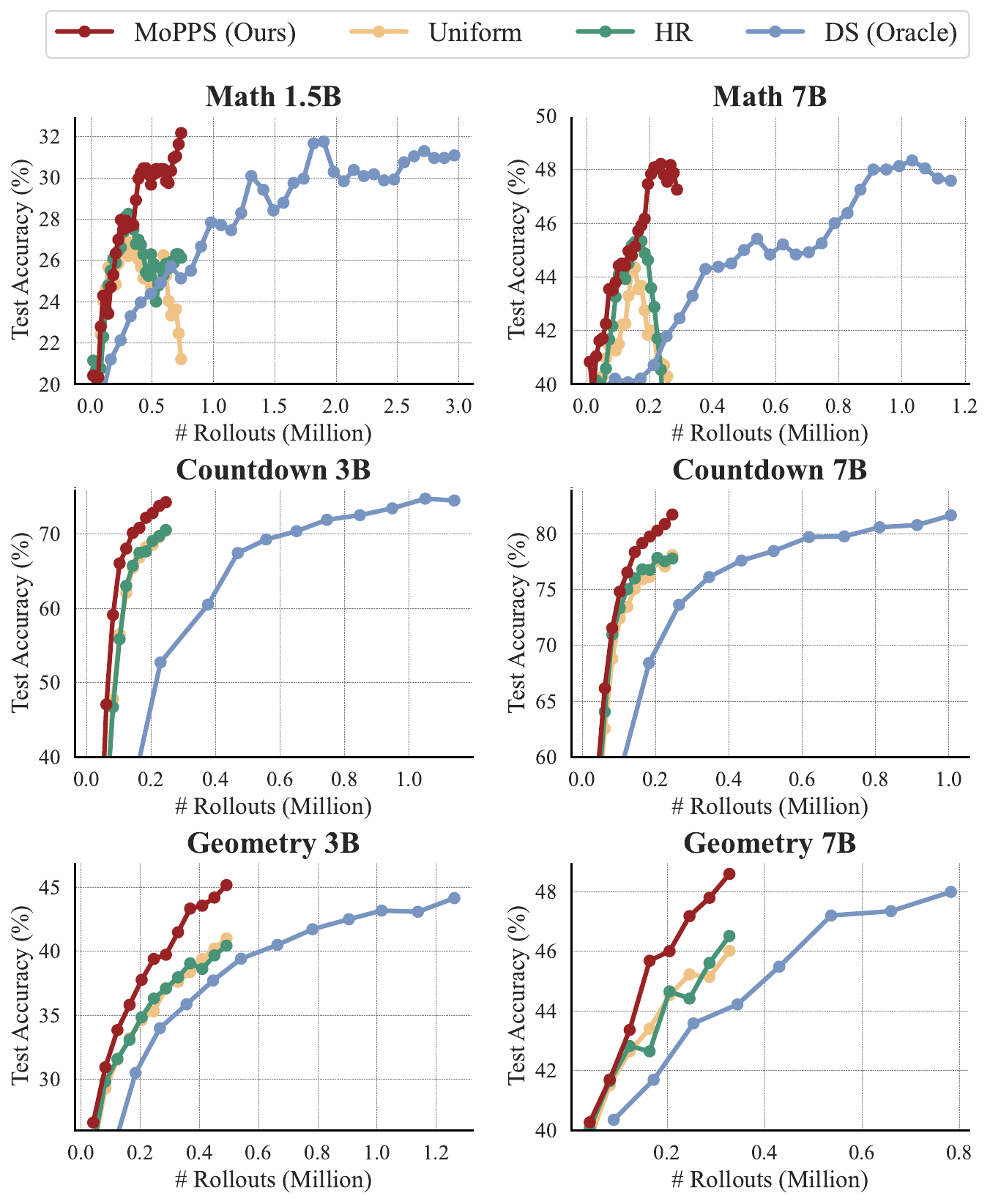}
    \vspace{-20pt}
    \caption{Training curves plotted against the number of rollouts generated by LLM during training. MoPPS achieves comparable accuracy with significantly fewer rollouts than DS.}
    \vspace{-15pt}
    \label{fig:rolloutcurve}
\end{figure}

\subsubsection{Rollout Efficiency.}
As shown in Fig.~\ref{fig:performance}, both MoPPS and DS accelerate training and improve final performance compared to Uniform.
While MoPPS matches DS in terms of training steps, this metric overlooks actual computational cost.
DS requires explicit evaluation of more prompts via over-sampling, incurring substantial LLM inference overhead.
To better reflect efficiency, we plot performance against the number of rollouts generated by LLM during training in Fig.~\ref{fig:rolloutcurve}.
MoPPS achieves comparable performance with far fewer rollouts, demonstrating superior sample efficiency.

\begin{figure*}[h]
    \centering
    \includegraphics[width=\linewidth]{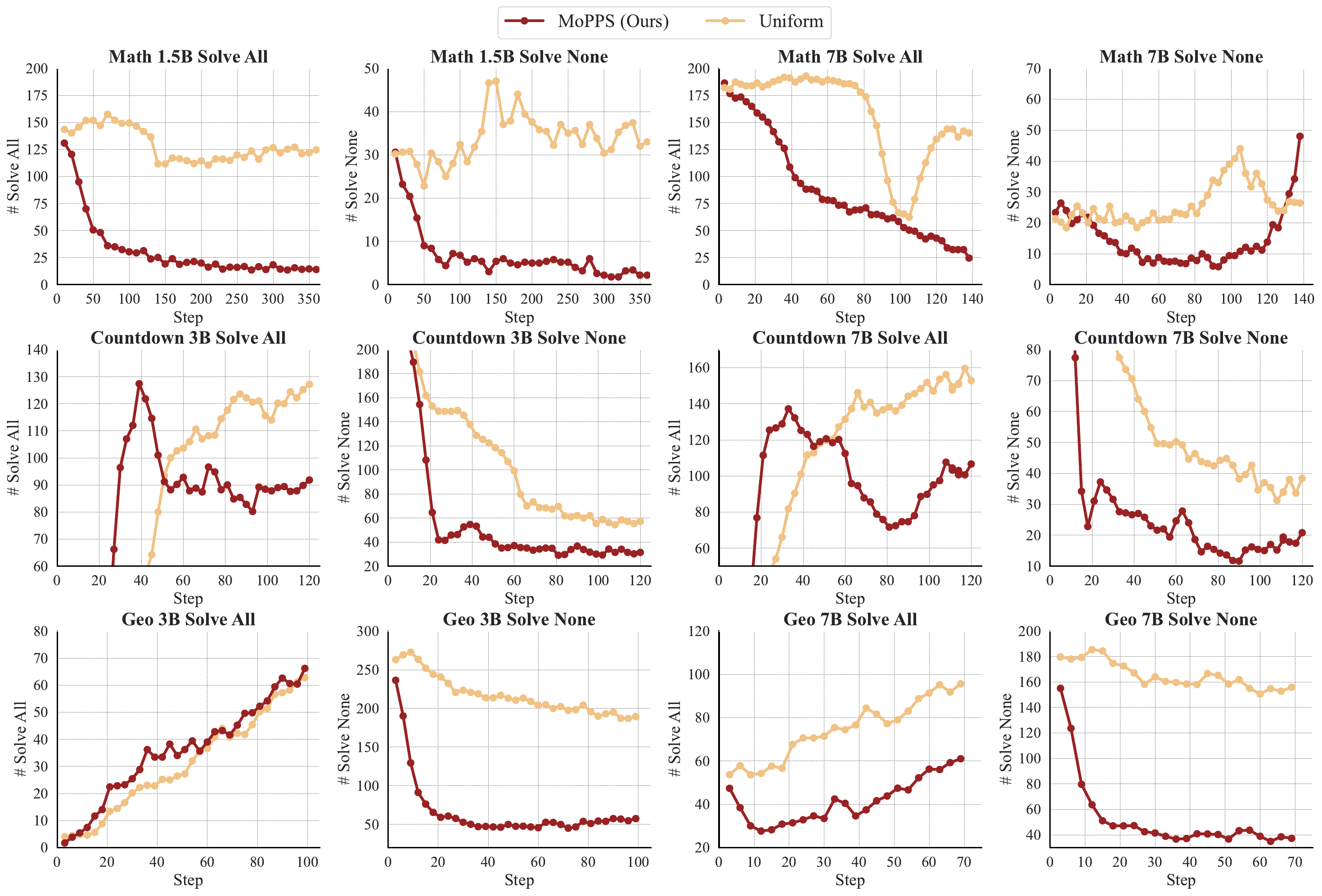}
    \vspace{-20pt}
    \caption{
Number of ineffective prompts (i.e., Solve-All or Solve-None) per batch.
MoPPS substantially reduces such prompts compared to uniform sampling, leading to more efficient training.
}
    \label{fig:databatch}
\end{figure*}

\subsubsection{Reduction of Ineffective Prompts.}
As noted in DAPO~\citep{yu2025dapo}, prompts that always succeed (Solve-All) or fail (Solve-None) lead to zero advantage in GRPO, resulting in no gradient for policy updates.
These prompts are therefore ineffective.  
DAPO mitigates this issue via dynamic sampling with hard evaluation, which is computationally expensive.  
In contrast, MoPPS amortizes this cost through lightweight posterior sampling.
To assess the effectiveness of MoPPS, we track the number of ineffective prompts per batch during training.  
As shown in Fig.~\ref{fig:databatch}, MoPPS significantly reduces the proportion of such prompts compared to uniform sampling, highlighting its benefit in improving training efficiency.

\begin{figure*}[h]
    \centering
    \includegraphics[width=\linewidth]{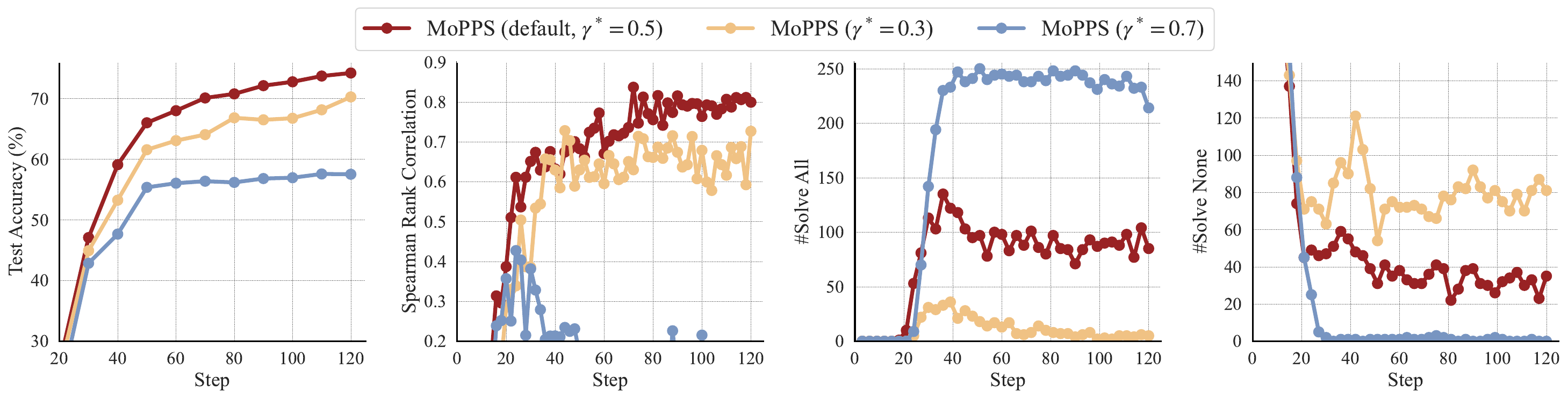}
    \vspace{-15pt}
    \caption{
Ablation study on target success rate $\gamma^*$. (a) Performance comparison under different $\gamma^*$ values on the Countdown task with Qwen2.5-3B. (b) Spearman rank correlation under different $\gamma^*$ values. (c,d) Number of ineffective prompts (i.e., Solve-All or Solve-None) per batch. These results support choosing intermediate success rates for stronger learning signals.
}
    \label{appfig:targetablate}
\end{figure*}
\subsubsection{Ablation Study on Target Success Rate.}
Prior studies~\citep{bae2025online, chen2025self} suggest that prompts with intermediate success rates provide stronger learning signals.
Based on this insight, we set the target success rate $\gamma^*$ to 0.5.
To empirically validate this choice, we conduct an ablation study on the Countdown task using Qwen2.5-3B.
As shown in Fig.~\ref{appfig:targetablate}(a,b), setting $\gamma^*$ to either 0.3 (favoring overly hard prompts) or 0.7 (favoring overly easy prompts) leads to degraded performance and less accurate difficulty predictions.
Examining the training batch composition in Fig.~\ref{appfig:targetablate}(c,d), we observe that both settings reduce the number of effective prompts. In particular, $\gamma^* = 0.7$ causes the batch to be dominated by Solve-All prompts, while $\gamma^* = 0.3$ overemphasizes unsolvable ones.
These results support the effectiveness of targeting prompts with intermediate success rates, i.e., $\gamma^* \approx 0.5$.
While adjusting $\gamma^*$ slightly around 0.5 may yield further gains, especially in the presence of estimation error, we leave the fine-tuning of this parameter to future work.

\begin{figure}[ht]
    \centering
    \includegraphics[width=\linewidth]{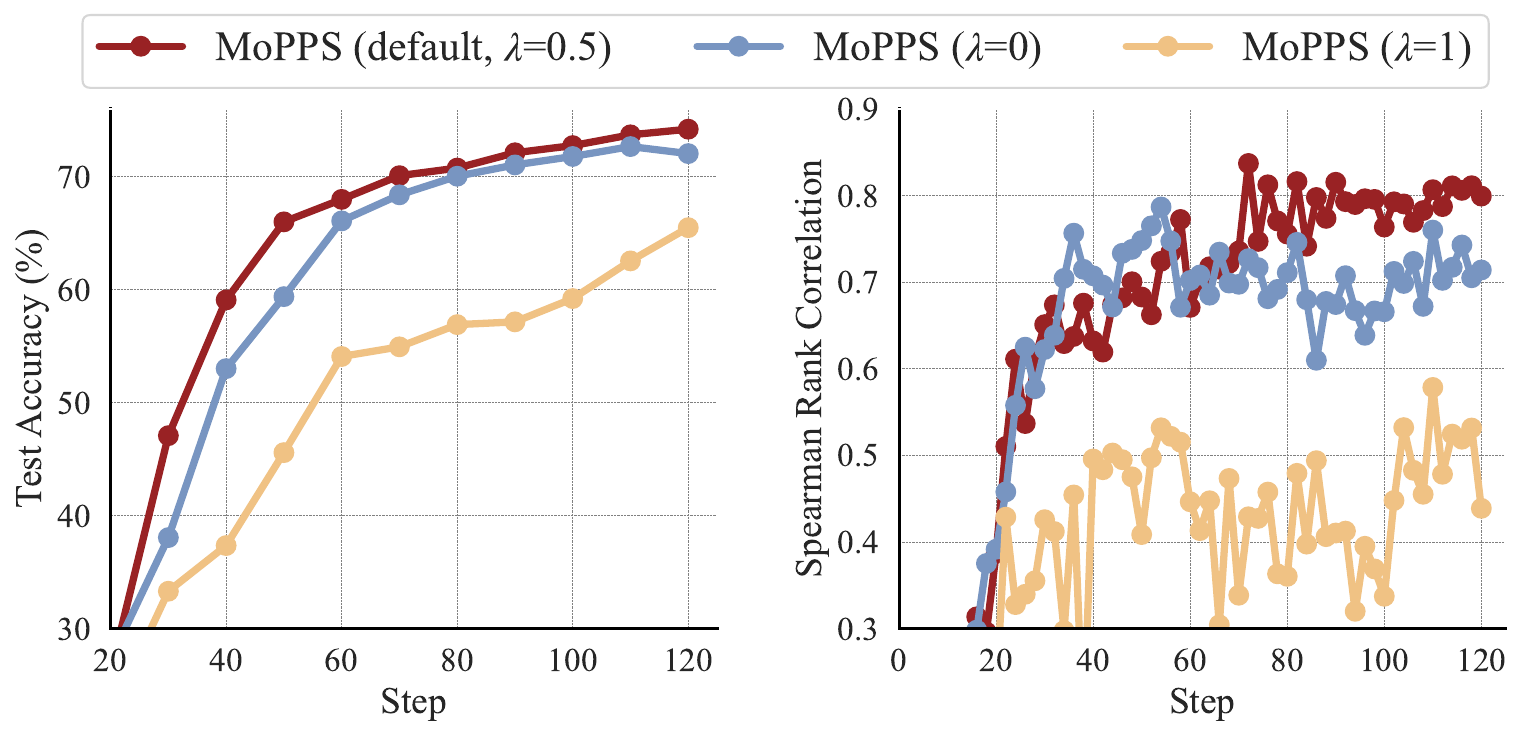}
    \vspace{-20pt}
    \caption{Ablation study on the temporal discounting strategy. We evaluate MoPPS under different $\lambda$ values, including disabling temporal discounting (MoPPS ($\lambda=1$)) and using only current feedback (MoPPS ($\lambda=0$)).}
    \vspace{-10pt}
    \label{fig:decayablate}
\end{figure}
\subsubsection{Ablation Study on Temporal Discounting.}\label{appsec:ablate}
To address nonstationarity during training, we introduce a temporal discounting strategy.
We conduct an ablation study on the Countdown task with Qwen2.5-3B to evaluate its effectiveness.
As shown in Fig.~\ref{fig:decayablate}, removing temporal discounting (\textit{MoPPS ($\lambda=1$)}) leads to a significant performance drop compared to the default MoPPS.
We also test an extreme setting where only current feedback is used (\textit{MoPPS ($\lambda=0$)}), which results in degraded performance and unstable estimation due to insufficient historical context.
In settings with weaker nonstationarity, where historical signals are more reliable, this extreme setup is expected to perform worse due to the lack of accumulated context.
Moreover, the Spearman correlation analysis shows that incorporating TD yields more reliable posterior estimation.
These results highlight the importance of temporal discounting for maintaining accuracy and robustness under non-stationary training dynamics.

\paragraph{Discussion.} 
Notably, the primary goal of this work is to demonstrate that prompt difficulty can be online predicted and leveraged for accelerated RL finetuning via appropriate selection strategies.
As such, the focus lies not on absolute performance tuning, but on validating the effectiveness of our predictive selection framework.
Even so, MoPPS already achieves strong and competitive results, matching Dynamic Sampling (DS) while requiring significantly fewer LLM rollouts.
The ablation studies further supports the soundness of the design and the proposed component.
We therefore leave finer hyperparameter tuning, e.g., the decay factor $\lambda$ and target success rate $\gamma^*$ to future work.
Moreover, in tasks like MATH where training dynamics are relatively stable, we disable discounting to simplify the implementation and tuning requirements.

\subsubsection{Ablation Study on Candidate Batch Size.}
We evaluate the effect of candidate batch size 
$\hat{\mathcal{B}}$ on Countdown, as shown in Table \ref{tab:candidate_size}.
The results show that performance improves as the candidate size increases, due to the expanded exploration space.
This is a key advantage of MoPPS: prior methods relying on real evaluations are limited by computational cost and cannot scale, whereas MoPPS uses lightweight prediction models to efficiently explore a much larger set of prompts. In practice, setting $\hat{\mathcal{B}}$ to the entire prompt pool is viable.

\begin{table}[t]
    \centering
    \caption{Performance with Different Candidate Batch Sizes on Countdown.}
    \label{tab:candidate_size}
    \renewcommand{\arraystretch}{1.2}
    \resizebox{\linewidth}{!}{
    \begin{tabular}{lcccc}
        \toprule
        \textbf{Candidate Size $\hat{\mathcal{B}}$} & \textbf{3B CD-34} & \textbf{3B CD-4} & \textbf{7B CD-34} & \textbf{7B CD-4} \\
        \midrule
        $1\times\mathcal{B}$ (Uniform) & 69.87 & 39.42 & 77.84 & 53.27 \\
        $2\times\mathcal{B}$  & 72.29 & 43.15 & 78.53 & 54.40 \\
        $4\times\mathcal{B}$  & 73.17 & 45.70 & 79.86 & 58.04 \\
        $\approx 8\times\mathcal{B}$ (Entire Pool) & \textbf{74.46} & \textbf{47.02} & \textbf{82.18} & \textbf{59.16} \\
        \bottomrule
    \end{tabular}
    }
\end{table}

\subsubsection{Response Length.}
Prior work~\citep{yu2025dapo} has shown that response length is closely tied to training stability and model performance.
Fig.~\ref{fig:responselength} tracks the mean response length during MATH training under different sampling strategies.
After an initial warm-up phase for posterior construction, MoPPS exhibits a similar trend to DS, i.e., generating consistently longer and more stable responses length than Uniform.
This partially explains their superior performance, as longer responses generally enable better exploration and more complex reasoning~\citep{yu2025dapo}.

\begin{figure}[h]
    \centering
    \includegraphics[width=\linewidth]{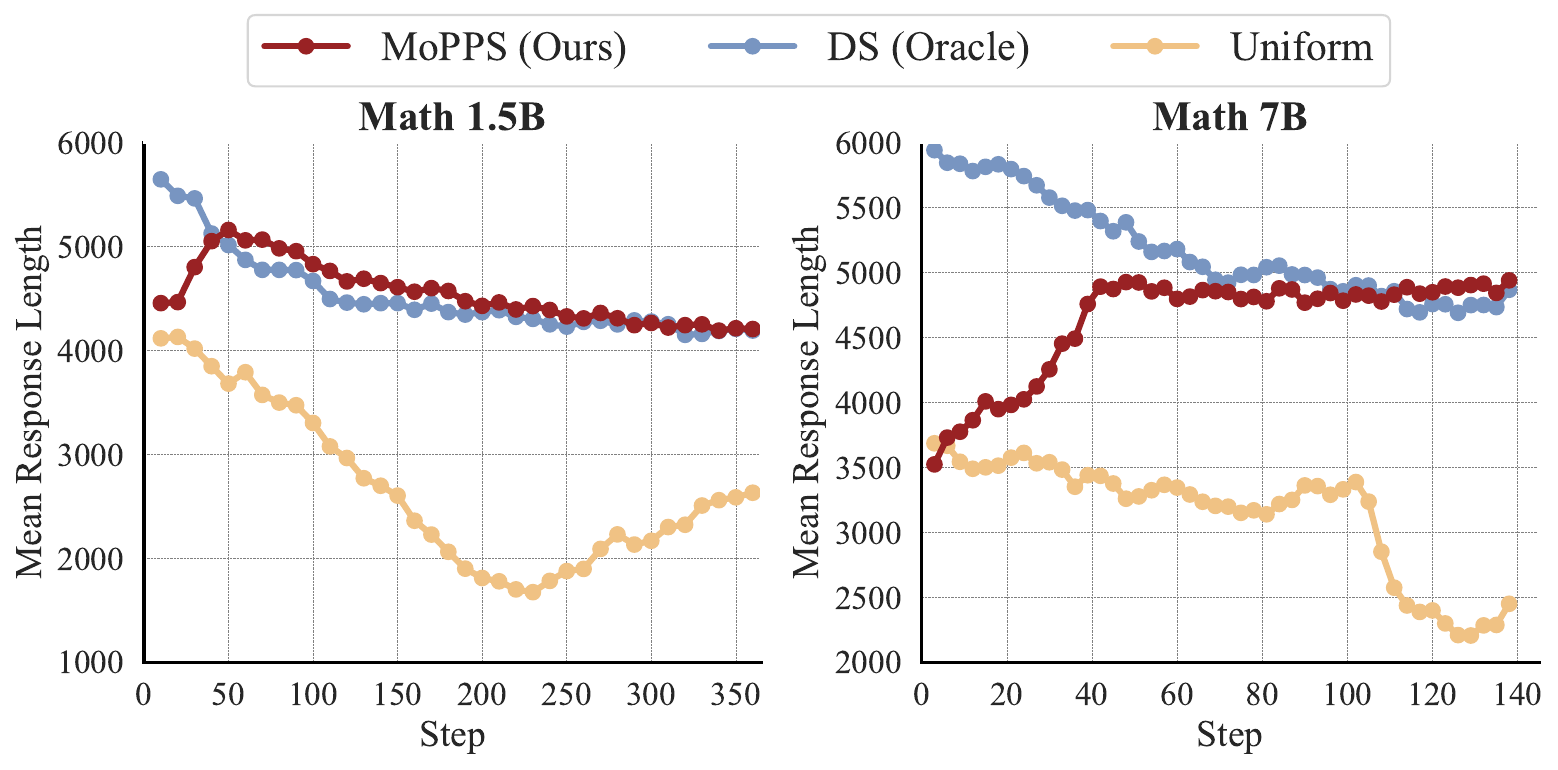}
    \vspace{-15pt}
    \caption{Mean response length during MATH training. MoPPS and DS both elicit longer responses than Uniform, explaining improved performance.}
    \label{fig:responselength}
\end{figure}

\subsubsection{Selected Prompt Length.}
We also analyze the average length of selected prompts during MATH training, as shown in Fig.~\ref{fig:promptlength}.
Both MoPPS and DS consistently prefer longer prompts compared to Uniform.
We hypothesize that longer prompts may be more complex with numerous conditions, encouraging the model to engage in deeper reasoning and reflection, which may lead to longer and more diverse responses~\citep{song2025fastcurl}.
A more detailed analysis of this observation is left for future work.
\begin{figure}[h]
    \centering
    \includegraphics[width=\linewidth]{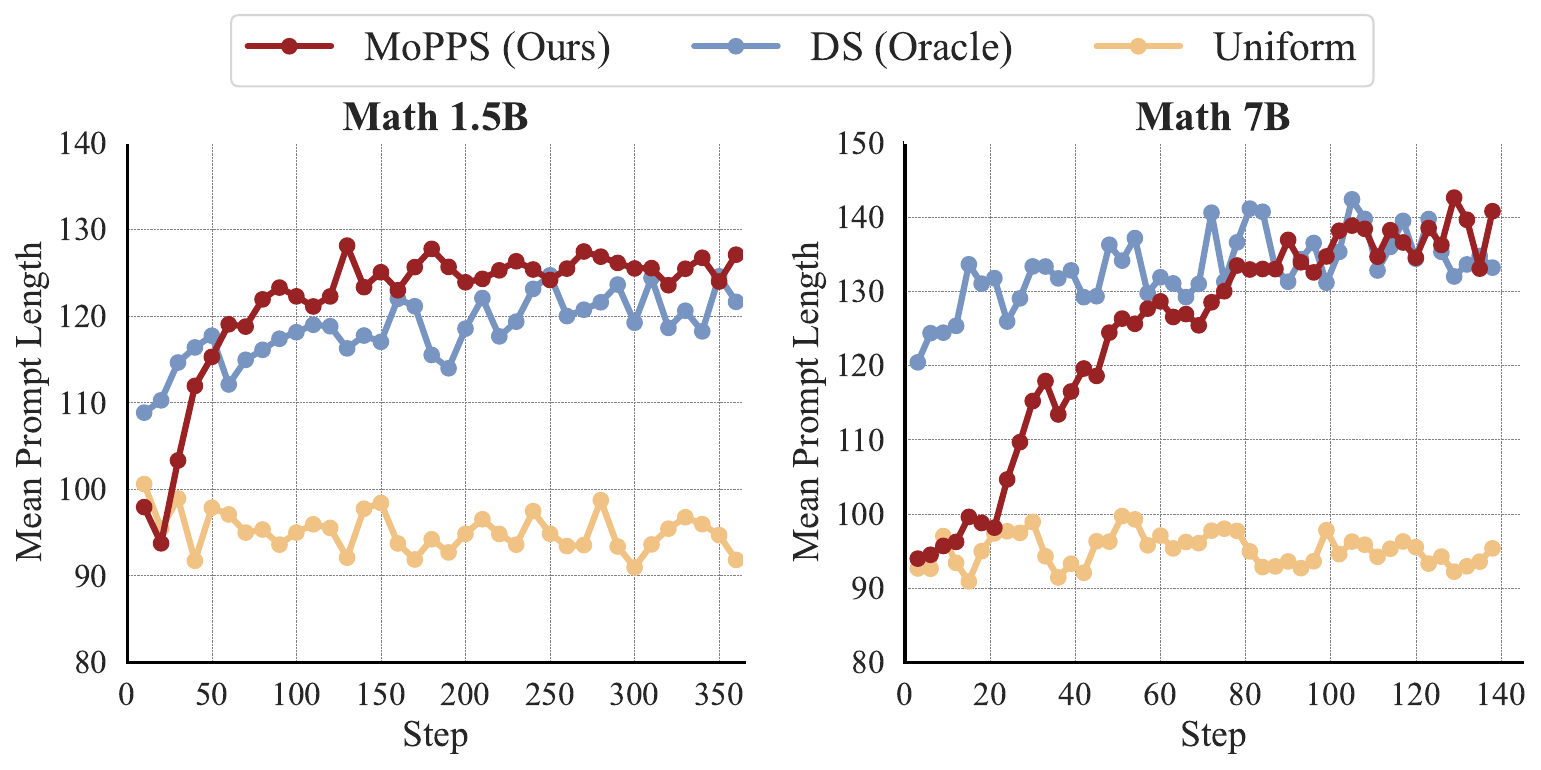}
    \vspace{-15pt}
    \caption{Mean prompt length during MATH training. MoPPS and DS tend to select longer prompts than Uniform.}
    \vspace{-10pt}
    \label{fig:promptlength}
\end{figure}

\section{Data Examples}
\label{appsec:dataexample}
The prompt templates for MATH and Geometry3k are adopted from the official verl framework, while the template for Countdown follows the format introduced in \cite{tinyzero}.

\begin{tcolorbox}[example, title=MATH example]
\textbf{Prompt:}

A rectangular band formation is a formation with $m$ band members in each of $r$ rows, where $m$ and $r$ are integers. A particular band has less than 100 band members. The director arranges them in a rectangular formation and finds that he has two members left over. If he increases the number of members in each row by 1 and reduces the number of rows by 2, there are exactly enough places in the new formation for each band member. What is the largest number of members the band could have? Let's think step by step and output the final answer within \texttt{\textbackslash boxed\{\}}.

\textbf{Ground-Truth Answer:}

98
\end{tcolorbox}
\begin{tcolorbox}[example, title=Countdown example]
\textbf{Prompt:}

A conversation between User and Assistant. The user asks a question, and the Assistant solves it. The assistant first thinks about the reasoning process in the mind and then provides the user with the answer.

User: Using the numbers [44, 19, 35], create an equation that equals 98. You can use basic arithmetic operations (+, -, *, /) and each number can only be used once. Show your work in $<$think$>\ </$think$>$ tags. And return the final answer in $<$answer$>\ </$answer$>$ tags, for example $<$answer$> (1 + 2) / 3 <$/answer$>$.

Assistant: Let me solve this step by step.

$<$think$>$

\textbf{Ground-Truth Answer:}

[44, 19, 35]
\end{tcolorbox}
\begin{tcolorbox}[example, title=Geometry3k example]
\textbf{Prompt:}

\includegraphics[width=0.2\linewidth]{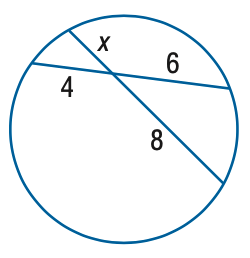}
Find x. You FIRST think about the reasoning process as an internal monologue and then provide the final answer. The reasoning process MUST BE enclosed within $<$think$>\ </$think$>$ tags. The final answer MUST BE put in \texttt{\textbackslash boxed\{\}}.

\textbf{Ground-Truth Answer:}

3
\end{tcolorbox}

\end{document}